\documentclass{article}
\usepackage{arxiv}
\usepackage{amsmath,amsfonts,bm}

\def\eqref#1{equation~\ref{#1}}
\def\1{\bm{1}}

\DeclareMathAlphabet{\mathsfit}{\encodingdefault}{\sfdefault}{m}{sl}
\SetMathAlphabet{\mathsfit}{bold}{\encodingdefault}{\sfdefault}{bx}{n}

\usepackage{xcolor}
\definecolor{lightblue}{rgb}{0.6, 0.8, 0.9}
\definecolor{darkblue}{HTML}{0c58ac}
\definecolor{darkgreen}{rgb}{0, 0.55, 0.12}
\definecolor{darkred}{rgb}{0.6,0,0}
\definecolor{yellowbg}{RGB}{255,249,224}
\definecolor{greenbg}{RGB}{232,246,232}
\definecolor{bluebg}{RGB}{226,241,255}
\usepackage[numbers]{natbib}
\usepackage[
  breaklinks=true,
  colorlinks,
  linkcolor=darkred,
  citecolor=darkgreen,
  bookmarks=false
]{hyperref}
\usepackage{url}

\usepackage{tikz}
\usepackage{graphicx}
\usepackage{wrapfig}
\usepackage{booktabs}
\usepackage{adjustbox}
\usepackage{multirow}
\usepackage{xspace}
\usepackage{array}
\usepackage{bm}
\usepackage{bbm}
\usepackage{xcolor}
\usepackage{amsthm, amsmath, amssymb}
\usepackage{cleveref}
\usepackage{algorithm}
\usepackage{algorithmic}
\newtheorem{definition}{Definition}
\newtheorem{theorem}{Theorem}
\newtheorem{remark}{Remark}

\usepackage{listings}
\usepackage{wrapfig}
\usepackage{caption}
\usepackage{subcaption}
\usepackage{colortbl}
\usepackage{adjustbox}
\usepackage{makecell}
\usepackage{pifont}
\usepackage{tcolorbox}
\usepackage{setspace}
\usepackage{fancybox}
\usepackage{tocloft}
\usepackage{etoc}
\usepackage{enumitem}
\usepackage{pifont}
\usepackage{bm}
\definecolor{darksilver}{RGB}{160,160,160}
\definecolor{mymauve}{rgb}{0.58,0,0.82}
\newlength{\mysize}

\usepackage{needspace}
\usepackage{placeins}
\newsavebox{\layoutleftbox}
\newsavebox{\layoutrightbox}
\input{tables/table_style}
\newcommand{\etal}{\textit{et al.}\xspace}

\allowdisplaybreaks

\newtheorem{thm}{Theorem}

\newtheorem{prop}[thm]{Proposition}

\theoremstyle{definition}

\newtheorem{assumption}[thm]{Assumption}

\title{Learning What to Trust in Multimodal Learning\\ under Noisy Supervision}

\author{Jiashuo Zou \qquad Xiaobo Xia\thanks{Corresponding author.}\\
University of Science and Technology of China\\
\tt\small zjs3490581068@gmail.com \quad xiaoboxia@ustc.edu.cn
}

\begin{document}
%\raggedbottom

\maketitle

\begin{abstract}
Multimodal classification processes and relates information from multiple modalities to achieve more accurate predictions. However, existing methods typically rely on high-quality ground-truth labels, which are difficult to obtain in real-world scenarios. While sample-selection methods for learning with noisy labels aim to identify correctly labeled examples from noisy data, traditional methods primarily focus on unimodal settings and fail to exploit multimodal information fully. This motivates us to build a more reliable noise detector in multimodal learning. To this end, we theoretically analyze the relationship between representation structure and noise detection capability. Based on this analysis, we propose REFINE, which is a multimodal label-noise detection framework that jointly uses fused and unimodal representations for label-noise detection. Specifically, REFINE constructs discriminative eigenvectors through discriminative analysis of the target and background classes and selects trusted representation spaces with better noise detection capability for each class. Within each trusted space, REFINE measures the alignment between each instance representation and the discriminative eigenvectors. It then combines the subsets selected from these spaces. The combined set provides cleaner supervision for updating the multimodal classifier, thereby reducing the influence of mislabeled examples during training and improving model generalization. Extensive experiments across diverse tasks demonstrate REFINE's superiority compared to baseline methods. The source code will be publicly available.

\end{abstract}
\addtocontents{toc}{\protect\setcounter{tocdepth}{-1}}

\section{Introduction}
\label{sec:introduction}

Humans perceive the world through a rich variety of
multimodal signals, including visual, auditory, and textual information.
Data from different modalities contain both shared semantics and
modality-specific information~\cite{dufumier2025align,zong2025self_supervised_multimodal}.
Inspired by this capability, multimodal learning aims to build models that
can process and relate information from multiple modalities~\cite{liu2026principled_multimodal}.
For classification, extensive research has explored cross-modal
interactions and information fusion to exploit complementary cues from
different modalities and improve classification
performance~\cite{xu2023multimodal_transformers,zhu2024vision_x,nagrani2021attention,qin2026deep}.
However, these methods typically rely on correct labels for
supervision~\cite{zhang2026tmnr,han2022trusted}.
Real-world data are typically annotated manually or collected through web
crawling so that their labels may be easily corrupted in practice~\cite{wei2021learning,song2022learning}, resulting in mislabeled data.
Recent studies have shown that even a small amount of such corrupted data can hinder
the generalization of deep neural networks owing to their strong memorization of noisy labels~\cite{kim2021fine,zhang2021understanding,yang2024searching}.
Hence, it becomes crucial to train more reliable multimodal classifiers
under noisy supervision.

\begin{wrapfigure}{r}{0.48\textwidth}
  \centering
  \includegraphics[width=\linewidth]{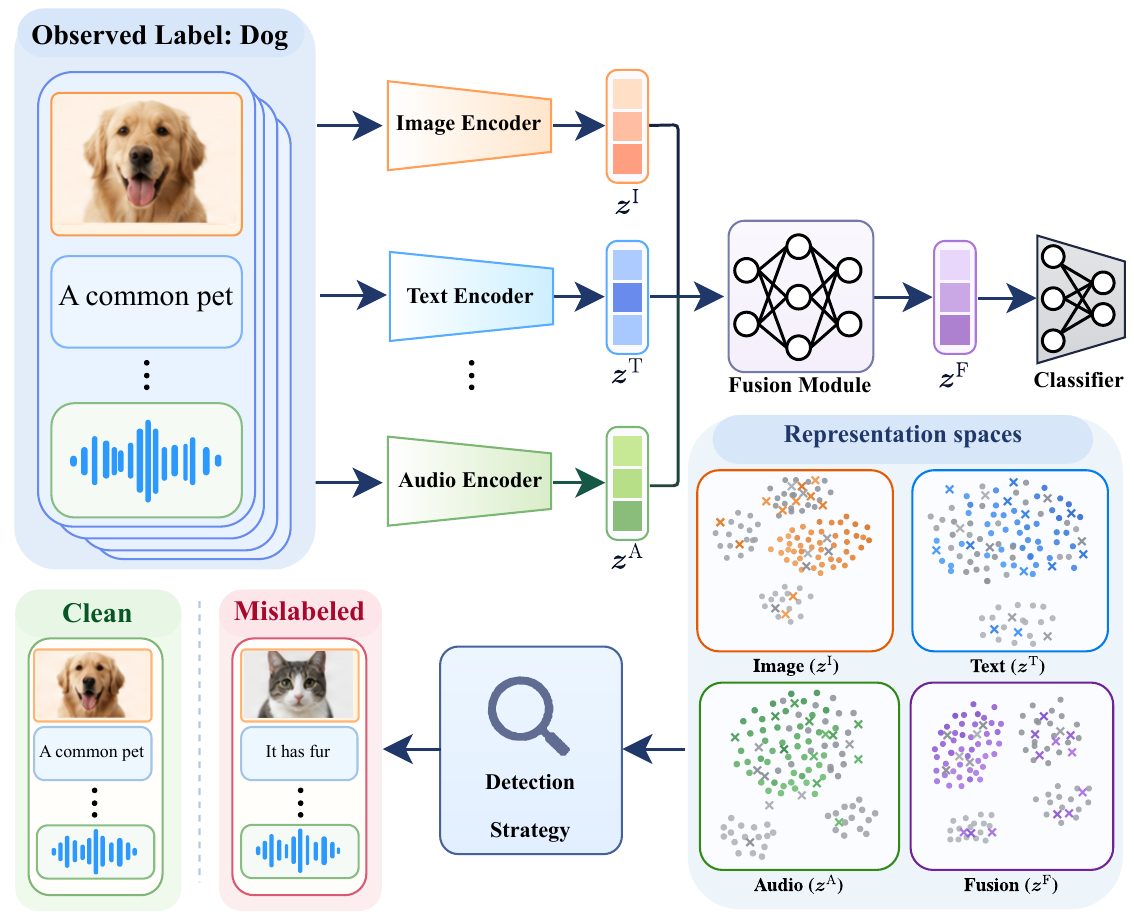}
  \caption{The illustration of mislabeled data detection in multimodal learning.
  The model needs to identify information in multimodal data that is
  useful for assessing label reliability and use this information to
  distinguish correctly labeled examples from mislabeled examples.}
  \label{fig:framework}
\end{wrapfigure}

Sample selection offers a simple, effective, and readily extensible
approach to handling noisy labels by identifying correctly labeled training
examples and limiting the use of incorrect labels for
supervision~\cite{han2018co,li2020dividemix,kim2021fine,wang2024knockoffs_spr,yuan2025enhancing}.
Unfortunately, existing sample-selection methods have primarily focused
on unimodal settings.
As illustrated in Fig.~\ref{fig:framework}, a multimodal model provides both
a fused representation space and multiple unimodal representation spaces,
which offer additional cues for assessing label reliability.
Relying solely on information obtained after modality fusion,
such as losses and gradients, for selection may overlook modality-specific cues
useful for assessing label reliability; vice versa.
This motivates us to delve into a new problem: \textit{how can we build a more
reliable label noise detector in multimodal learning?}
This problem encompasses two key objectives: \textit{1) identifying multimodal
information with better label-noise detection capability; 2) using this
information to detect label noise more reliably.}

In this work, we take representation-based sample-selection methods~\cite{zhu2022detecting,kim2021fine}
as our starting point because they can naturally use both fused and unimodal
representation spaces as sources of information for sample selection.
We theoretically analyze the relationship between representation structure
and label-noise detection capability and argue that a reliable basis
for selection should reflect the discriminative information of
the target class relative to other classes.
Further analysis shows that the noise detection capabilities of
different representation spaces vary across classes,
motivating us to select better representation spaces for each class.

Based on this analysis, we propose a reliable multimodal noise detection
framework named \textit{\underline{re}liable sample \underline{fi}ltering i\underline{n}
multimodal l\underline{e}arning} (REFINE).
The framework is supported by theoretical analysis and adaptively selects
reliable examples without requiring predefined selection thresholds,
noise-rate estimation, or additional clean examples.
First, we perform discriminative analysis of the target and background
classes in each representation space to construct discriminative
eigenvectors.
These eigenvectors emphasize target-class structure and reduce the
influence of shared background.
Next, using the projections of example representations onto these
eigenvectors, we compare the fused space and all unimodal spaces pairwise.
For each class, we select a set of trusted spaces based on their ability
to distinguish the target class from the background classes.
Finally, we apply label agreement in each trusted space, retaining an
example only if its alignment score along its observed class's
discriminative eigenvector exceeds the corresponding scores for all other
classes.
We combine the examples retained by the trusted spaces to obtain the final
selected subset. The resulting set provides more reliable supervision for training the
multimodal classifier, mitigating the effect of mislabeled examples and improving generalization.

\textbf{Contributions.} Before delving into details, we summarize our contributions as follows:
\begin{itemize}
    \item We study the problem of  label-noise detection in multimodal learning.
    This problem concerns how to identify information in multimodal data
    that is useful for assessing label reliability and use it to build a
    more effective noise detection framework.
    The goal is to provide reliable supervision for multimodal
    classification under noisy supervision.
    \item We propose REFINE (\textit{\underline{re}liable sample \underline{fi}ltering i\underline{n}
multimodal l\underline{e}arning}) as a method for multimodal noise detection that
    incorporates fused and unimodal representation spaces into a
    unified sample-selection process. The method combines target–background
    discriminative analysis, trusted space selection, and label agreement to
    adaptively select reliable samples. It requires no predefined
    selection thresholds, noise-rate estimation, or additional clean examples.
    We also theoretically characterize the relationship between
    representation structure and label-noise detection capability,
    providing a theoretical basis for constructing discriminative
    selection criteria and determining trusted representation spaces.
    \item We conduct experiments on seven datasets covering different
    modality combinations.
    The results on all seven datasets demonstrate REFINE's superior
    classification performance under multiple types of noise.
    Comprehensive analyses, including sample-selection quality evaluation,
    component ablations, and representation-space analysis, further
    validate the effectiveness and design rationale of REFINE.
    Experiments integrating REFINE with different
    learning-with-noisy-labels training paradigms also validate its ability to improve
    performance as a general plug-in sample-selection mechanism.
\end{itemize}

\section{Related Work}
\label{sec:related-work}

\subsection{Learning with Noisy Labels}
Learning with noisy labels~\cite{angluin1988learning} aims to mitigate the influence of potentially
incorrect observed labels on model training, and improve generalization performance~\cite{song2022learning,han2020survey,shi2024survey}. Existing methods broadly include
noise-tolerance losses or training regularization~\cite{wang2019symmetric,zhou2026correntropy,gonzalez2025robust,ma2020normalized,li2025rml++,yuan2026instance,hu2019simple}, noise-transition modeling~\cite{xia2019anchor,li2024transferring,guo2022simt,xia2022extended,liao2025instance}, and
label correction~\cite{tu2023learning,zheng2020error,jiang2024tackling,englesson2024robust}, sample selection~\cite{yuan2025enhancing,yang2024searching,xia2024regularly,yao2021jo,feng2024clipcleaner}, and semi-supervised learning methods
that regard unselected samples as unlabeled data~\cite{li2020dividemix,fooladgar2024manifold,kim2025splitnet,wei2026robust,liu2026cpln}.
In this work, we focus on sample selection because it can directly connect label reliability with learned representation structure, making it particularly suitable for exploiting multimodal representations.

Note that although these general methods of learning with noisy labels can, in principle, be applied to fused multimodal representations, such a direct transfer usually treats the fused representation as an \textit{ordinary single input} and fails to fully exploit the complementary information across different representation spaces. Among sample-selection methods, FINE~\cite{kim2021fine} estimates sample reliability from the latent structure within each observed class. However, it operates in a single representation space. In multimodal learning, both fused and unimodal representations are available, and their ability to distinguish clean from noisy samples may vary across classes. This motivates us to extend representation-based sample selection to explicitly exploit and adaptively select among multimodal representation spaces.

\subsection{Multimodal Representation Learning and Classification}
Multimodal representation learning integrates heterogeneous information, \textit{e.g.}, images, text, audio, and video, to obtain more informative representations than those from a single modality~\cite{xu2023multimodal_transformers,zhu2024vision_x,zong2025self_supervised_multimodal,zhang2023generalization,cicchetti2025gramian,zhang2024multimodal,liang2021multibench}. According to the stage at which fusion occurs, existing multimodal classification methods
can generally be categorized into paradigms such as feature-level fusion and decision-level
fusion~\cite{xu2023multimodal_transformers,zhu2024vision_x}. This work focuses on the widely
adopted feature-level fusion paradigm. Namely, different modalities are first encoded by their
respective encoders, then projected into representation spaces with compatible dimensions,
and combined into a fused representation before the classifier.

Recent studies have begun to address noisy supervision in multi-view and multimodal learning~\cite{liang2023multiview_lnl,xu2025noisy_label_calibration}. More specifically, existing methods may rely on cross-view consistency and evidence aggregation~\cite{zhang2026tmnr}, additional clean data~\cite{jiang_robust_2024}, noise-transition modeling~\cite{wang_molar_2026}, or small-loss sample selection~\cite{wu2026multiasnet,liu_minjot_2024}. However, these methods may assume that each modality is individually predictive, require additional clean supervision, or rely heavily on classifier predictions that can themselves be affected by label noise. In contrast, our method exploits both fused and unimodal representation spaces for reliable sample selection without requiring additional clean samples or prior knowledge of the noise rate.
\section{Preliminaries}
\label{sec:preliminary}

\subsection{Problem Formulation}
\label{subsec:problem-formulation}

We consider a $K$-class multimodal classification problem with noisy labels.
Given a multimodal dataset
$\mathcal D=\{(\bm x_i,\widetilde y_i)\}_{i=1}^{N}$ containing $N$ examples,
$\bm x_i=(\bm x_i^1,\ldots,\bm x_i^M)$ is an instance comprising $M$
modalities, $\widetilde y_i\in\{1,\ldots,K\}$ is the observed label that
may have been corrupted, and $y_i\in\{1,\ldots,K\}$ denotes the true class
label. Our goal is to select a subset of examples from $\mathcal D$ with
reliable observed labels and use this subset to train a more robust
multimodal classifier.

We adopt feature-level fusion for multimodal representation learning. For
the $m$-th modality, an encoder $h^m$ followed by a projection layer $f^m$
maps $\bm x_i^m$ to
$\bm z_i^m=f^m\bigl(h^m(\bm x_i^m)\bigr)$. The fusion module $\Phi$
aggregates the modality-specific representations into
$\bm z_i^{\mathrm F}=\Phi\bigl(\bm z_i^1,\ldots,\bm z_i^M\bigr)$,
and the classification head generates predictions solely from the fused
representation $\bm z_i^{\mathrm F}$. To jointly denote the fused and
unimodal spaces, let
\begin{equation}
    \mathcal R=\{\mathrm F,1,\ldots,M\},
    \qquad
    \bm z_i^r\in\mathbb R^{d^r},
    \quad r\in\mathcal R.
    \label{eq:representation-spaces}
\end{equation}
Here, $r=\mathrm F$ corresponds to the fused representation, $r=m$ to the
$m$-th unimodal representation, and $d^r$ denotes the dimension of the
corresponding representation space.

To analyze the sample-selection capability of different representation
spaces, we consider a binary classification setting with true labels
$y\in\{-1,+1\}$ in the subsequent theoretical analysis. Without loss of
generality, we focus on the set of examples whose observed label is
$\widetilde y=+1$. To establish our theorems, we adopt the following
reasonable assumptions, motivated by related work~\cite{kim2021fine,lee2019robust,natarajan2013learning,xia2023combating}.

\begin{assumption}
\label{assump:gaussian-representations}
In any space $r\in\mathcal R$, the representations of instances with true
labels $y=+1$ and $y=-1$ follow
$
        \bm z^r\mid y=+1
        \sim\mathcal N(\bm v^r,\sigma^2\bm I),
        \bm z^r\mid y=-1
        \sim\mathcal N(\bm w^r,\sigma^2\bm I),
$
respectively, where
$\lVert\bm v^r\rVert_2=\lVert\bm w^r\rVert_2=1$.
\end{assumption}

\begin{assumption}
\label{assump:class-conditional-corruption}
Label corruption depends only on the true label, and clean examples
constitute a majority in both observed classes.
\end{assumption}

\subsection{FINE for Learning with Noisy Labels: A Recap}
\label{subsec:fine-sample}

FINE~\cite{kim2021fine} assesses label reliability through the alignment
between an instance representation and the leading eigenvector of its
observed class. For the general $K$-class task, fix a space
$r\in\mathcal R$ and define the class matrix
\begin{equation}
    \bm A_k^r
    =\frac{1}{N_k}\sum_{i:\widetilde y_i=k}
    \bm z_i^r(\bm z_i^r)^\top
    \in\mathbb R^{d^r\times d^r},
    \label{eq:class-second-moment}
\end{equation}
where $N_k$ is the number of examples in observed class $k$. Let
$\operatorname{TopEig}$ return a unit eigenvector associated with the
largest eigenvalue. FINE computes the alignment score using
\begin{equation}
        \bm u_{k,\mathrm{FINE}}^r
        =\operatorname{TopEig}(\bm A_k^r),
        \qquad s_i^r
        =\bigl[(\bm u_{k,\mathrm{FINE}}^r)^\top\bm z_i^r\bigr]^2,
        \qquad \widetilde y_i=k.
    \label{eq:fine-alignment-score}
\end{equation}
FINE then fits a two-component Gaussian mixture model
(GMM)~\cite{mclachlan2019finite} to the alignment scores within each observed class
and selects the examples associated with the component having the larger
mean as the clean subset. FINE performs this selection procedure in a single representation space,
whereas multimodal models produce both fused and unimodal representations.
When FINE is applied to this setting, the choice of representation remains
an open question.

\subsection{Limitations of FINE in Multimodal Formulation}
\label{subsec:fine-limitations}

We analyze how representations affect the sample-selection capability of
FINE under the model in Section~\ref{subsec:problem-formulation}.

\begin{definition}[Class Angle]
\label{def:class-angle}
For the two unit mean directions $\bm v^r$ and $\bm w^r$ in representation
space $r$, define
\begin{equation}
    \theta^r
    =\angle(\bm w^r,\bm v^r)
    =\arccos\bigl|(\bm v^r)^\top\bm w^r\bigr|
    \in[0,\pi/2].
    \label{eq:class-angle}
\end{equation}
\end{definition}
This angle characterizes the geometric relationship between the two class
mean directions.

\begin{definition}[Alignment-Score Difference]
\label{def:expected-alignment-score-difference}
For any unit direction $\bm u\in\mathbb R^{d^r}$ in space $r$, define
\begin{equation}
    \begin{aligned}
        \Delta^r(\bm u)
        &={}
        \mathbb E_{\bm z^r\sim\mathcal N(\bm v^r,\sigma^2\bm I)}
        \bigl[\langle\bm u,\bm z^r\rangle^2\bigr]
        -\mathbb E_{\bm z^r\sim\mathcal N(\bm w^r,\sigma^2\bm I)}
        \bigl[\langle\bm u,\bm z^r\rangle^2\bigr]=(\bm u^\top\bm v^r)^2-(\bm u^\top\bm w^r)^2.
    \end{aligned}
    \label{eq:expected-alignment-score-difference}
\end{equation}
\end{definition}
\begin{tcolorbox}
    [colframe=gray!80, colback=gray!7, boxrule=0.2pt, arc=2pt, boxsep=2.5pt, left=2pt, right=2pt, top=2pt, bottom=2pt]
\begin{remark}
$\Delta^r(\bm u)$ is the difference between the mean alignment scores of
clean and mislabeled examples along direction $\bm u$, and characterizes
their separation along the same direction. We prove in the Appendix that
this score difference is an important factor in deriving lower bounds on
recall and precision, both of which increase with the score difference.
\end{remark}
\end{tcolorbox}
\begin{theorem}[FINE Alignment-Score Differences Across Representation Spaces]
\label{thm:fine-alignment-score-difference}
Under the model in Section~\ref{subsec:problem-formulation}, the
alignment-score difference produced by FINE in space $r\in\mathcal R$ is
\begin{equation}
    \Delta_{\mathrm{FINE}}^r
    :=\Delta^r(\bm u_{\mathrm{FINE}}^r)
    =\frac{(1-\tau)\sin^2\theta^r}
    {\sqrt{(1-\tau)^2+4\tau\cos^2\theta^r}},
    \label{eq:fine-expected-alignment-score-difference}
\end{equation}
where $\tau$ is the ratio of mislabeled to clean examples, i.e.,
$\frac{\mathbb P(y=-1\mid\widetilde y=+1)}
{\mathbb P(y=+1\mid\widetilde y=+1)}$.
For any fixed $\tau\in[0,1)$, $\Delta_{\mathrm{FINE}}^r$ is strictly
increasing in $\theta^r\in[0,\pi/2]$. For any fixed
$\theta^r\in(0,\pi/2)$, it is strictly decreasing in $\tau\in[0,1)$.
In particular, because all representation spaces share the same $\tau$,
for any $r,s\in\mathcal R$,
\begin{equation}
    \Delta_{\mathrm{FINE}}^{r}
    >\Delta_{\mathrm{FINE}}^{s}
    \quad\Longleftrightarrow\quad
    \theta^{r}>\theta^{s}.
    \label{eq:fine-space-ordering}
\end{equation}
\end{theorem}
\begin{tcolorbox}
    [colframe=gray!80, colback=gray!7, boxrule=0.2pt, arc=2pt, boxsep=2.5pt, left=2pt, right=2pt, top=2pt, bottom=2pt]
\begin{remark}
Theorem~\ref{thm:fine-alignment-score-difference} demonstrates that, for the same
observed class, a larger $\theta^r$ makes the mean alignment scores of
clean and mislabeled examples easier to separate. Therefore, the relative
advantages of fused and unimodal spaces depend on the representation
structure of the specific class, and no type of representation space can
be presumed optimal for all classes. Based on this relationship,
Section~\ref{sec:method} enlarges the score difference in two complementary
ways: by optimizing class directions within each space and selecting
representations appropriate for each class across spaces.
\end{remark}
\end{tcolorbox}

\section{Method}
\label{sec:method}

\begin{figure}[!htbp]
    \centering
    \includegraphics[width=\textwidth]{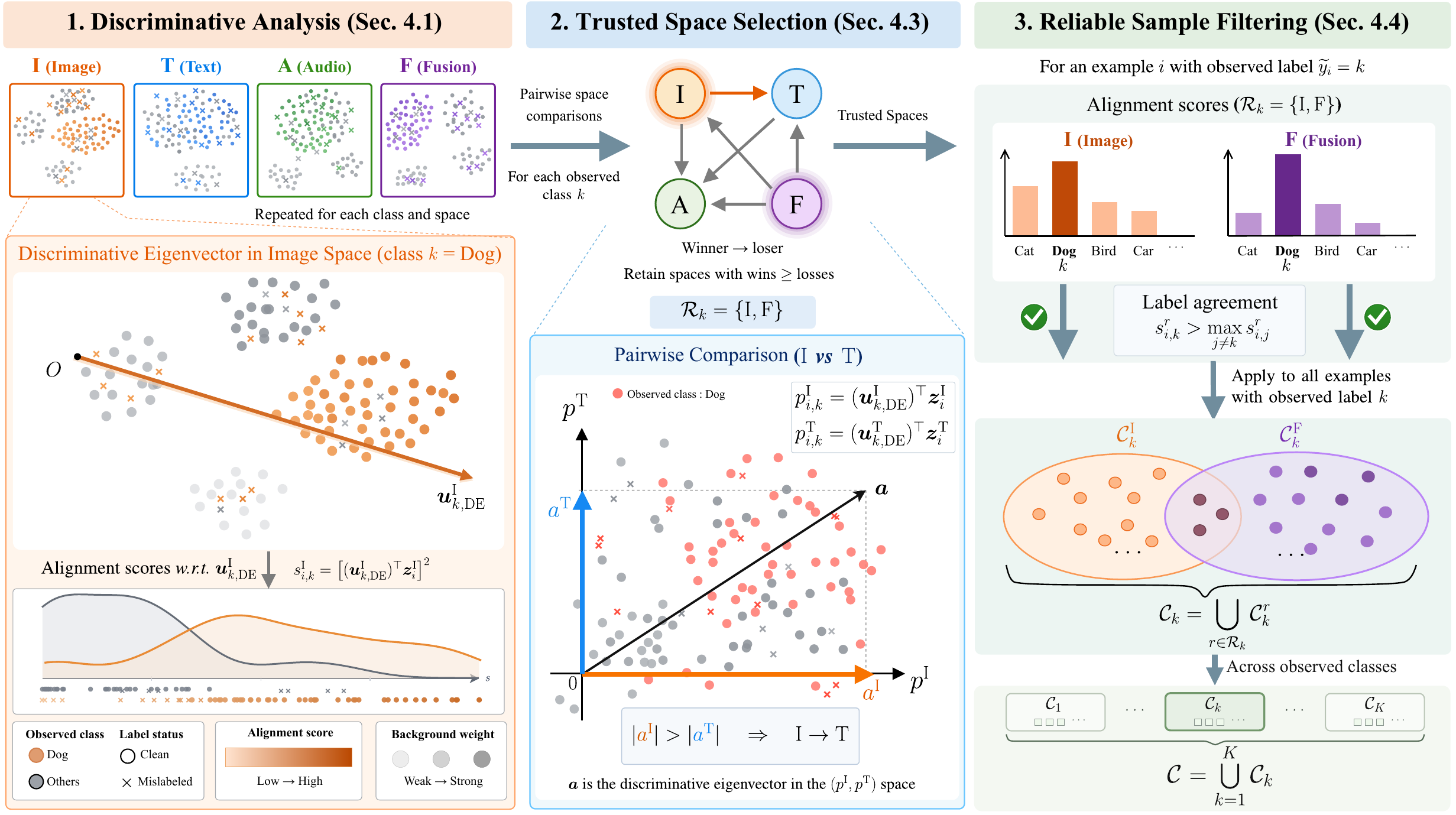}
    \caption{The overall framework of the proposed REFINE. (\textit{Part 1}) We construct discriminative eigenvectors
    for each class in the fused and unimodal spaces through discriminative
    analysis of the target and background. (\textit{Part 2}) We select trusted spaces with
    stronger noise detection ability for each class through pairwise comparisons
    across spaces. (\textit{Part 3}) We select reliable examples based on their alignment
    with the discriminative eigenvectors in each trusted space and take the union
    of the selected subsets.}
    \label{fig:overview}
\end{figure}

In this section, we formally propose REFINE, which
uses fused and unimodal representations for reliable sample filtering.
First, discriminative analysis of the target and background yields
\textit{discriminative eigenvectors} (abbreviated as DE) for each class. Next, \textit{trusted space
selection} (abbreviated as TSS) determines a set of trusted spaces for each class. Finally, \textit{label agreement} (abbreviated as LA) selects
examples, and the selected subsets from the trusted spaces are combined
by taking their union. Fig.~\ref{fig:overview} presents the overall
procedure. Technical details are provided below.

\subsection{Discriminative Analysis}
\label{subsec:discriminative-analysis}
Note that the leading eigenvector of FINE~\cite{kim2021fine} maximizes the mean alignment score of the
target observed class without considering other classes. If other classes
also have large alignment scores along the same vector, this eigenvector
may primarily reflect structure shared across classes, weakening the
separation between clean and mislabeled examples. Inspired by the
target-background discriminative criterion of
dPCA~\cite{chen2019dpca}, we construct discriminative eigenvectors that align strongly
with the target class and weakly with the background.

\textbf{Discriminative eigenvector.}
To focus on the background on directional components shared with the target
class, we assign weights according to the alignment scores of non-target
classes along the target's leading direction. Specifically, for target
observed class $k$ in space $r$, we compute the mean alignment score of
each non-target observed class $j\neq k$ along the target class's FINE
eigenvector:
\begin{equation}
    d_{kj}^r
    =(\bm u_{k,\mathrm{FINE}}^r)^\top
    \bm A_j^r\bm u_{k,\mathrm{FINE}}^r,
    \qquad j\neq k.
    \label{eq:background-alignment-score}
\end{equation}
We normalize these scores into background weights and construct the
weighted background matrix
\begin{equation}
        \gamma_{kj}^r
        =\frac{d_{kj}^r}{\sum_{\ell\neq k}d_{k\ell}^r},
        \qquad \bm B_k^r
        =\sum_{j\neq k}\gamma_{kj}^r\bm A_j^r.
    \label{eq:weighted-background}
\end{equation}
Non-target classes with stronger alignment receive larger background
weights. Afterward, the corresponding discriminative eigenvector is
\begin{equation}
    \bm u_{k,\mathrm{DE}}^r
    =\operatorname*{arg\,max}_{\lVert\bm u\rVert_2=1}
    \frac{\bm u^\top\bm A_k^r\bm u}
    {\bm u^\top(\bm B_k^r+\epsilon\bm I)\bm u}.
    \label{eq:de-eigenvector}
\end{equation}
DE maximizes the ratio of the target alignment score to the background
alignment score. To address possible rank deficiency of the background
matrix $\bm B_k^r$, we add a positive constant $\epsilon>0$ to its diagonal
entries, following dPCA~\cite{chen2019dpca}. The alignment score of instance
$i$ for class $k$ is then
\begin{equation}
    s_{i,k}^r
    =\bigl[(\bm u_{k,\mathrm{DE}}^r)^\top\bm z_i^r\bigr]^2.
    \label{eq:de-alignment-score}
\end{equation}

\subsection{Theoretical Analysis of DE}
\label{subsec:de-analysis}

We first analyze how DE changes the alignment-score difference under the
model in Section~\ref{subsec:problem-formulation}, and then explain the
geometric meaning of the background weights in the multiclass setting.

\begin{theorem}[DE Enlarges the Alignment-Score Difference]
\label{thm:de-gap}
Under the model in Section~\ref{subsec:problem-formulation}, if
$\epsilon=0$, then
\begin{equation}
        \Delta_{\mathrm{DE}}^r
        :=\Delta^r(\bm u_{\mathrm{DE}}^r)
        =\sin\theta^r
        \sqrt{1-\frac{\cos^2\theta^r}{(1+2\sigma^2)^2}}
        \geq\Delta_{\mathrm{FINE}}^r.
    \label{eq:de-gap}
\end{equation}
The inequality is strict when $0<\theta^r<\pi/2$, and
$\Delta_{\mathrm{DE}}^r-\Delta_{\mathrm{FINE}}^r$ strictly increases with
$\tau$. For fixed $\sigma>0$, $\Delta_{\mathrm{DE}}^r$ is strictly
increasing in $\theta^r\in[0,\pi/2]$.
\end{theorem}
\begin{tcolorbox}
    [colframe=gray!80, colback=gray!7, boxrule=0.2pt, arc=2pt, boxsep=2.5pt, left=2pt, right=2pt, top=2pt, bottom=2pt]
\begin{remark}
Theorem~\ref{thm:de-gap} shows that incorporating background information
increases the alignment-score difference between clean and mislabeled
examples within the same space, and this advantage grows as contamination
in the target observed class increases.
\end{remark}
\end{tcolorbox}
In multiclass settings, the background consists of multiple non-target
classes. The following proposition explains how the background weights are
assigned according to the geometry of the target class.

\begin{prop}[Angular Ordering of Background Weights]
\label{prop:background-weight-ordering}
Fix target class $k$ and space $r$. Suppose that each non-target observed
class satisfies
$\bm A_j^r=\bm v_j^r(\bm v_j^r)^\top+\sigma^2\bm I$, where
$\lVert\bm v_j^r\rVert_2=1$ and $\sigma>0$. Let
$\phi_{kj}^r=\arccos\bigl|(\bm u_{k,\mathrm{FINE}}^r)^\top
\bm v_j^r\bigr|$. For any $j,\ell\neq k$,
\begin{equation}
    d_{kj}^r=\sigma^2+\cos^2\phi_{kj}^r,
    \qquad \phi_{kj}^r<\phi_{k\ell}^r
    \iff\gamma_{kj}^r>\gamma_{k\ell}^r.
    \label{eq:background-weight-ordering}
\end{equation}
\end{prop}
\begin{tcolorbox}
    [colframe=gray!80, colback=gray!7, boxrule=0.2pt, arc=2pt, boxsep=2.5pt, left=2pt, right=2pt, top=2pt, bottom=2pt]
\begin{remark}
A smaller $\phi_{kj}^r$ means that class $j$ has a larger alignment score
along the target class's leading direction and hence receives a larger
weight in $\bm B_k^r$. This makes discriminative analysis focus more on
background structure that overlaps with the target class.
\end{remark}
\end{tcolorbox}
\subsection{Trusted Space Selection}
\label{subsec:trusted-modality-selection}

Theorem~\ref{thm:fine-alignment-score-difference} shows that fused and
unimodal spaces differ in their ability to distinguish clean from
noisy examples. At the same time, each space may contain information
unavailable in other spaces. Relying on a single space for sample selection
may lose this complementary information, whereas unconditionally trusting
any space may allow a less discriminative space to influence the decision.
Therefore, for each class, we aim to select a set of trusted representation
spaces from the candidate set $\mathcal R$, which contains the fused and
unimodal spaces, accounting for both their complementarity and
discriminative ability.
To compare the discriminative ability of two distinct spaces
$r,s\in\mathcal R$ for observed class $k$, we form a two-dimensional
representation from the projections of the same instance onto their
respective discriminative eigenvectors:
\begin{equation}
        p_{i,k}^r
        =(\bm u_{k,\mathrm{DE}}^r)^\top\bm z_i^r,
        \qquad \bm p_{i,k}^{r,s}
        =\begin{bmatrix}p_{i,k}^r\\p_{i,k}^s\end{bmatrix}
        \in\mathbb R^2.
    \label{eq:paired-de-projections}
\end{equation}

We apply the discriminative analysis in
Section~\ref{subsec:discriminative-analysis} again in this two-dimensional
space. Specifically, we construct the second-moment matrix of each
observed class,
$\bm C_{j,k}^{r,s}
    =\frac{1}{N_j}\sum_{i:\widetilde y_i=j}
    \bm p_{i,k}^{r,s}(\bm p_{i,k}^{r,s})^\top,$
where $k$ specifies the class eigenvector used for projection and $j$
specifies the observed class of the examples used in the calculation.
We take the leading eigenvector of the target matrix $\bm C_{k,k}^{r,s}$
and compute the background weights using
Eq.~(\ref{eq:background-alignment-score}) and Eq.~(\ref{eq:weighted-background}) to
obtain the weighted background $\bm D_k^{r,s}$. The corresponding
two-dimensional discriminative eigenvector is
\begin{equation}
    \bm a_k^{r,s}
    =\operatorname*{arg\,max}_{\lVert\bm a\rVert_2=1}
    \frac{\bm a^\top\bm C_{k,k}^{r,s}\bm a}
    {\bm a^\top(\bm D_k^{r,s}+\epsilon\bm I)\bm a}.
    \label{eq:two-dimensional-de-eigenvector}
\end{equation}
For the current pair of spaces, write its coordinates as
$\bm a_k^{r,s}=(a_k^r,a_k^s)^\top$. These coordinates are the coefficients
in the combined projection $a_k^r p_{i,k}^r+a_k^s p_{i,k}^s$; a larger
magnitude assigns greater weight to the projection from the corresponding
space. The following result relates this comparison to the alignment-score
differences in the individual spaces.

\begin{theorem}[Ordering of Coefficients and Alignment-Score Differences]
\label{thm:two-dimensional-de-ordering}
Under the model in Section~\ref{subsec:problem-formulation}, if
$\epsilon=0$, then
\begin{equation}
        |a_k^r|>|a_k^s|
        \quad\Longleftrightarrow\quad
        \Delta_{\mathrm{DE}}^r>\Delta_{\mathrm{DE}}^s
        \quad\Longleftrightarrow\quad
        \theta^r>\theta^s.
    \label{eq:two-dimensional-de-ordering}
\end{equation}
\end{theorem}
\begin{tcolorbox}
    [colframe=gray!80, colback=gray!7, boxrule=0.2pt, arc=2pt, boxsep=2.5pt, left=2pt, right=2pt, top=2pt, bottom=2pt]
\begin{remark}
Theorem~\ref{thm:two-dimensional-de-ordering} shows that, in discriminative
analysis of spaces, the space with the larger coefficient magnitude has a
larger alignment-score difference, providing a basis for comparing the
discriminative ability of the two spaces.
\end{remark}
\end{tcolorbox}
In practice, we compare all unordered pairs of spaces in $\mathcal R$ for each class. The
space with the strictly larger coordinate magnitude records a win, and the
other records a loss. Equal magnitudes count as neither a win nor a loss
for either space. Let $W_k^r$ and $L_k^r$ denote the numbers of wins and
losses, respectively, for space $r$ on class $k$. The trusted space set is
\begin{equation}
    \mathcal R_k=\{r\in\mathcal R:W_k^r\geq L_k^r\}.
    \label{eq:trusted-space-set}
\end{equation}
Each comparison without a tie contributes one win and one loss. Therefore, at
least one space satisfies this condition. Fused and unimodal spaces follow
the same selection rule, and the resulting combination of trusted spaces
varies across classes.

\subsection{Reliable Sample Filtering}
\label{subsec:trusted-sample-selection}

Prior work fits a GMM to alignment scores within each observed class to
distinguish clean from mislabeled examples. However, a relatively high
score within the observed class does not imply that the instance is most
closely aligned with that class. Assessing the reliability of an observed
label therefore also requires considering its alignment scores for other
classes.

\textbf{Label agreement.}
We directly compute an instance's alignment scores along the
discriminative eigenvectors of all classes and retain the example when its
observed class has the strictly highest score. The set of examples selected
by space $r$ from observed class $k$ is
\begin{equation}
        \mathcal C_k^r
        =\bigl\{ (\bm x_i,\widetilde y_i)\in\mathcal D\mid
        \widetilde y_i=k,
        s_{i,k}^r>\max_{j\neq k}s_{i,j}^r\bigr\}.
    \label{eq:label-agreement-selection}
\end{equation}
This criterion selects examples according to the support for their
observed labels relative to other classes, eliminating classwise GMM
fitting and its probability threshold. The final selected subset is the
union of the subsets selected by the trusted spaces of each class:
\begin{equation}
    \mathcal C
    =\bigcup_{k=1}^{K}\bigcup_{r\in\mathcal R_k}\mathcal C_k^r.
    \label{eq:trusted-selection-union}
\end{equation}
An example is retained if it satisfies LA in at least one trusted space
for its observed class. Different trusted spaces can select examples
missed by one another, while TSS determines which spaces participate in
the final decision for each class.
Algorithm~\ref{alg:refine-selection} summarizes the complete
sample-selection procedure.

\begin{algorithm}[t]
    \caption{REFINE}
    \label{alg:refine-selection}
    \small
    \renewcommand{\algorithmicrequire}{\textbf{Input:}}
    \renewcommand{\algorithmicensure}{\textbf{Output:}}
    \begin{algorithmic}[1]
        \raggedright
        \REQUIRE Noisy multimodal dataset $\mathcal D$, multimodal model,
            \mbox{constant $\epsilon>0$}.
        \ENSURE Selected subset of examples $\mathcal C$.
        \STATE Initialize $\mathcal C\leftarrow\varnothing$.
        \STATE Extract representations
            $\{\bm z_i^r:i=1,\ldots,N,\ r\in\mathcal R\}$ for all
            instances using the multimodal model.
        \FOR{$k=1,\ldots,K$}
            \item[] \textit{/* (1) Construct class DE in each space */}
            \FOR{$r\in\mathcal R$}
                \STATE Compute background weights and $\bm B_k^r$ by
                    Eq.~(\ref{eq:background-alignment-score}) and Eq.~(\ref{eq:weighted-background}).
                \STATE Obtain $\bm u_{k,\mathrm{DE}}^r$ by
                    Eq.~(\ref{eq:de-eigenvector}).
            \ENDFOR
            \item[] \textit{/* (2) Select trusted spaces for each class */}
            \FORALL{unordered pairs $\{r,s\}\subseteq\mathcal R$, $r\neq s$}
                \STATE Construct two-dimensional representations
                    $\{\bm p_{i,k}^{r,s}\}_{i=1}^{N}$ for all instances by
                    Eq.~(\ref{eq:paired-de-projections}).
                \STATE Solve Eq.~(\ref{eq:two-dimensional-de-eigenvector})
                    to obtain $\bm a_k^{r,s}=(a_k^r,a_k^s)^\top$.
                \STATE Record a win for the space with the larger
                    coordinate magnitude and a loss for the other.
            \ENDFOR
            \STATE Form $\mathcal R_k$ from spaces with at least as many
                wins as losses.
        \ENDFOR
        \item[] \textit{/* (3) Form the selected subset of examples */}
        \FOR{$k=1,\ldots,K$}
            \FOR{$r\in\mathcal R_k$}
                \STATE Compute
                    $s_{i,j}^r=[(\bm u_{j,\mathrm{DE}}^r)^\top\bm z_i^r]^2$
                    for all $j=1,\ldots,K$ and all $i$ with
                    $\widetilde y_i=k$.
                \STATE Select examples to form $\mathcal C_k^r$ by
                    Eq.~(\ref{eq:label-agreement-selection}).
            \ENDFOR
        \ENDFOR
        \STATE Obtain $\mathcal C$ by taking the union in
            Eq.~(\ref{eq:trusted-selection-union}).
    \end{algorithmic}
\end{algorithm}

\section{Experiments}
\label{sec:experiments}
In this section, we conduct experiments to address the following research question:
\begin{itemize}
    \item \textbf{RQ1}: Does REFINE consistently improve multimodal classification under diverse label-noise conditions and training settings?
    \item \textbf{RQ2}: How do the proposed sample-selection components and representation-space selection contribute to the reliability of clean-sample selection?
    \item \textbf{RQ3}: Can REFINE serve as a general plug-in sample-selection mechanism across different learning-with-noisy-labels training paradigms?
\end{itemize}
\subsection{Experimental Setup}
\label{subsec:experimental-setup}
\textbf{Datasets.} We use six multimodal datasets for the experiments on \textit{synthetic} label noise, and one multimodal dataset for the experiments on \textit{realistic} label noise. Note that these comprise four image-text datasets, a video-audio dataset, an
audio-text dataset, and a text-video-audio dataset. For the synthetic experiments, we exploit UPMC-Food101~\cite{wang2015recipe}, N24News~\cite{wang2022n24news}, NWPU-Captions~\cite{cheng2022nwpu}, Rakuten France~\cite{amoualian2021commerce}, VGGSound50~\cite{qin2026deep,chen2020vggsound}, and MIntRec2.0~\cite{zhang2024mintrec2}. For the realistic one, we utilize BSD\footnote{\url{https://dcase.community/challenge2026/task-heterogeneous-audio-classification}}, which includes a noisy training set and an expert-annotated clean test set. The key statistics of the datasets are provided in Appendix\ref{app:datasets}.

\textbf{Label noise generation.} We consider four types of synthetic label noise. (1) \textit{Symmetric noise}~(abbreviated as Sym)~\cite{van2015learning}: each true label
is randomly flipped to another class with equal probability. (2) \textit{Asymmetric noise}~(abbreviated as Asym)~\cite{scott2013classification}: the labels of selected examples are
replaced with designated related classes according to a predefined semantic confusion mapping (\textit{e.g.}, waffles$\rightarrow$pancakes). (3) \textit{Full-modality instance-dependent noise}~(abbreviated as F-IDN)~\cite{zhang2026tmnr}: we follow TMNR by using the uncertainty of a full-modality
evidential classifier to determine corruption propensity and class evidence
to assign incorrect labels. (4) \textit{Partial-modality instance-dependent noise}~(abbreviated as P-IDN): to
simulate annotators referring to only some modalities during
annotation, we assign each instance a modality combination
drawn uniformly from a candidate set comprising the full
modality combination and all subsets obtained by omitting
one modality. We then generate noisy labels following the F-IDN rules using an evidential classifier trained for the selected combination.  All experiments on synthetic-noise datasets employ Sym. at a 50\% noise level,
Asym. at a 40\% noise level, and F-IDN and P-IDN at both 40\% and 50\% noise levels.

\begin{table}[!t]
  \caption{Classification accuracy (\%, mean$\pm$std.) of the proposed method and baseline
  methods on different datasets under different noise types and ratios. The
  ``Noise'' columns show the noise type and ratio, and the ``Average'' column
  reports the mean across the six noise settings. The best and second-best
  results are highlighted in \textbf{boldface} and \underline{underlined}, respectively.}
  \label{tab:b1-main-results}
  \centering
  \small
  \renewcommand{\arraystretch}{1.12}
  \newcommand{\acc}[2]{#1\,\textpm\,#2}
  \newcommand{\bestacc}[2]{\textbf{#1\,\textpm\,#2}}
  \newcommand{\secondacc}[2]{\underline{#1\,\textpm\,#2}}
  \settowidth{\dimen0}{\textbf{NWPU-Captions}}
  \settowidth{\dimen2}{DIST+CT}
  \settowidth{\dimen4}{\bestacc{00.00}{0.00}}
  \settowidth{\dimen6}{\textbf{F-IDN 50\%}}
  \setlength{\tabcolsep}{\dimexpr
    (\textwidth-\dimen0-\dimen2-3\dimen4-4\dimen6-3\arrayrulewidth)/18\relax}
  \setlength{\aboverulesep}{0pt}
  \setlength{\belowrulesep}{0pt}
  \newlength{\bOneinset}
  \setlength{\bOneinset}{1.8pt}
  \newcommand{\bOnetopgap}{0pt}
  \newcommand{\bOnebottomgap}{\bOneinset}
  \newcommand{\bOnegaps}[2]{\noalign{%
    \gdef\bOnetopgap{#1}\gdef\bOnebottomgap{#2}}}
  \newcommand{\bOneseparator}{\smash{\rule[
    \dimexpr-\dp\csname @arstrutbox\endcsname+\bOnebottomgap\relax
  ]{\arrayrulewidth}{\dimexpr
    \ht\csname @arstrutbox\endcsname+\dp\csname @arstrutbox\endcsname-
    \bOnetopgap-\bOnebottomgap\relax}}}
  \newcommand{\bOneheaderseparator}{\smash{\rule[
    -\dimexpr\dp\csname @arstrutbox\endcsname+\cmidrulewidth\relax
  ]{\arrayrulewidth}{\dimexpr
    \ht\csname @arstrutbox\endcsname+\dp\csname @arstrutbox\endcsname+
    \cmidrulewidth-\bOneinset\relax}}}
  \newlength{\refineoffset}
  \setlength{\refineoffset}{\dimexpr\dimen0+2\tabcolsep+\arrayrulewidth\relax}
  \newcommand{\refinerowshade}{\noalign{\begingroup
    \color{gray!15}%
    \vskip\bOneinset
    \hbox{\hskip\refineoffset\vrule width\dimexpr\textwidth-\refineoffset\relax
      height0pt depth\dimexpr
        \ht\csname @arstrutbox\endcsname+\dp\csname @arstrutbox\endcsname-2\bOneinset\relax}%
    \vskip-\dimexpr
      \ht\csname @arstrutbox\endcsname+\dp\csname @arstrutbox\endcsname-\bOneinset\relax
    \endgroup}}
  \begin{tabular}{c!{\bOneseparator}c!{\bOneseparator}cccccc!{\bOneseparator}c}
    \toprule
    \multicolumn{1}{c!{\bOneheaderseparator}}{\tableheaderspan{\textbf{Dataset}}} &
    \multicolumn{1}{c!{\bOneheaderseparator}}{\tableheaderspan{\textbf{Method}}} &
    \multicolumn{6}{c!{\bOneheaderseparator}}{\tableheadercell{\textbf{Noise}}} &
    \tableheaderspan{\textbf{Average}} \\
    \cmidrule(lr){3-8}
    & & \tableheadercell{\textbf{Sym 50\%}} & \tableheadercell{\textbf{Asym 40\%}} &
    \tableheadercell{\textbf{F-IDN 40\%}} & \tableheadercell{\textbf{F-IDN 50\%}} &
    \tableheadercell{\textbf{P-IDN 40\%}} & \tableheadercell{\textbf{P-IDN 50\%}} & \\
    \midrule

    \bOnegaps{\bOneinset}{0pt}
    \multirow{8}{*}{\textbf{UPMC-Food101}}
    & Standard & \acc{82.75}{0.33} & \acc{67.86}{0.22} & \acc{84.25}{0.10} & \acc{82.50}{0.11} & \acc{84.09}{0.08} & \acc{82.11}{0.07} & \acc{80.59}{0.05} \\
    \bOnegaps{0pt}{0pt}
    & CRUST & \acc{84.81}{0.16} & \secondacc{81.93}{0.23} & \acc{83.11}{0.25} & \acc{81.75}{0.04} & \acc{82.19}{0.33} & \acc{80.79}{0.20} & \acc{82.43}{0.08} \\
    & AUM & \acc{83.18}{0.12} & \acc{66.08}{0.56} & \acc{83.80}{0.13} & \acc{82.09}{0.30} & \acc{82.01}{0.23} & \acc{79.80}{0.22} & \acc{79.49}{0.15} \\
    & L2D & \acc{84.49}{0.22} & \acc{72.89}{0.37} & \secondacc{85.75}{0.21} & \acc{84.06}{0.11} & \secondacc{85.46}{0.22} & \secondacc{83.83}{0.35} & \acc{82.75}{0.07} \\
    & DIST & \secondacc{85.56}{0.20} & \acc{80.97}{0.38} & \acc{85.61}{0.23} & \secondacc{84.29}{0.15} & \acc{84.37}{0.17} & \acc{82.29}{0.15} & \secondacc{83.85}{0.10} \\
    & DIST+CT & \secondacc{85.56}{0.20} & \acc{80.96}{0.38} & \acc{85.61}{0.22} & \secondacc{84.29}{0.11} & \acc{84.37}{0.17} & \acc{82.31}{0.12} & \secondacc{83.85}{0.08} \\
    \bOnegaps{0pt}{\bOneinset}
    & FINE & \acc{84.78}{0.25} & \acc{69.08}{0.56} & \acc{84.70}{0.11} & \acc{83.57}{0.15} & \acc{83.04}{0.16} & \acc{81.75}{0.11} & \acc{81.15}{0.12} \\
    \cmidrule{2-9}
    \bOnegaps{\bOneinset}{\bOneinset}
    \refinerowshade
    & \textbf{REFINE}
    & \bestacc{86.63}{0.22}
    & \bestacc{85.28}{0.34}
    & \bestacc{86.85}{0.08}
    & \bestacc{85.84}{0.13}
    & \bestacc{86.41}{0.16}
    & \bestacc{85.12}{0.21}
    & \bestacc{86.02}{0.05} \\
    \midrule

    \bOnegaps{\bOneinset}{0pt}
    \multirow{8}{*}{\textbf{MIntRec2.0}}
    & Standard & \acc{28.48}{0.76} & \secondacc{28.90}{0.35} & \acc{35.41}{0.73} & \acc{34.87}{0.75} & \acc{35.93}{0.72} & \acc{34.34}{0.94} & \secondacc{32.99}{0.31} \\
    \bOnegaps{0pt}{0pt}
    & CRUST & \acc{29.79}{1.16} & \acc{27.34}{1.07} & \acc{34.90}{0.64} & \acc{34.36}{0.85} & \acc{35.10}{1.19} & \acc{33.63}{1.73} & \acc{32.52}{0.74} \\
    & AUM & \acc{27.55}{0.74} & \acc{25.15}{0.99} & \acc{31.60}{0.84} & \acc{30.25}{1.07} & \acc{31.82}{1.01} & \acc{29.90}{0.69} & \acc{29.38}{0.55} \\
    & L2D & \acc{27.75}{1.32} & \acc{25.50}{0.88} & \acc{32.72}{0.87} & \acc{30.28}{0.56} & \acc{33.21}{0.86} & \acc{31.02}{1.04} & \acc{30.08}{0.73} \\
    & DIST & \secondacc{30.04}{0.71} & \acc{25.76}{0.64} & \secondacc{35.75}{0.45} & \secondacc{35.22}{0.76} & \bestacc{36.24}{0.85} & \secondacc{34.43}{1.09} & \acc{32.91}{0.42} \\
    & DIST+CT & \secondacc{30.04}{0.71} & \acc{25.67}{0.49} & \secondacc{35.75}{0.45} & \secondacc{35.22}{0.76} & \bestacc{36.24}{0.85} & \secondacc{34.43}{1.09} & \acc{32.89}{0.43} \\
    \bOnegaps{0pt}{\bOneinset}
    & FINE & \acc{29.73}{1.11} & \acc{25.95}{0.70} & \acc{33.77}{1.15} & \acc{33.24}{1.39} & \acc{34.57}{0.40} & \acc{32.14}{0.60} & \acc{31.57}{0.53} \\
    \cmidrule{2-9}
    \bOnegaps{\bOneinset}{\bOneinset}
    \refinerowshade
    & \textbf{REFINE}
    & \bestacc{31.92}{0.85}
    & \bestacc{29.61}{1.09}
    & \bestacc{36.54}{1.34}
    & \bestacc{35.70}{0.40}
    & \secondacc{35.96}{0.91}
    & \bestacc{34.94}{1.81}
    & \bestacc{34.11}{0.78} \\
    \midrule

    \bOnegaps{\bOneinset}{0pt}
    \multirow{8}{*}{\textbf{NWPU-Captions}}
    & Standard & \acc{98.30}{0.16} & \acc{77.23}{0.98} & \acc{97.21}{0.17} & \acc{95.01}{0.51} & \acc{98.13}{0.27} & \acc{96.71}{0.49} & \acc{93.77}{0.28} \\
    \bOnegaps{0pt}{0pt}
    & CRUST & \secondacc{99.19}{0.03} & \acc{94.61}{0.22} & \acc{97.14}{0.40} & \acc{95.77}{0.59} & \acc{97.45}{0.32} & \acc{96.48}{0.35} & \acc{96.78}{0.22} \\
    & AUM & \acc{98.90}{0.22} & \acc{77.31}{0.93} & \acc{98.27}{0.17} & \acc{96.56}{0.50} & \acc{97.69}{0.20} & \acc{96.15}{0.34} & \acc{94.15}{0.19} \\
    & L2D & \acc{98.63}{0.14} & \acc{83.30}{2.38} & \acc{98.39}{0.17} & \acc{97.14}{0.50} & \secondacc{98.48}{0.35} & \secondacc{97.71}{0.23} & \acc{95.61}{0.36} \\
    & DIST & \acc{99.14}{0.07} & \secondacc{97.18}{0.15} & \acc{98.58}{0.19} & \acc{97.09}{0.46} & \acc{97.90}{0.26} & \acc{96.65}{0.29} & \secondacc{97.76}{0.14} \\
    & DIST+CT & \acc{99.14}{0.07} & \acc{97.16}{0.17} & \acc{98.58}{0.19} & \acc{97.09}{0.46} & \acc{97.90}{0.26} & \acc{96.65}{0.29} & \acc{97.75}{0.15} \\
    \bOnegaps{0pt}{\bOneinset}
    & FINE & \acc{98.84}{0.17} & \acc{88.01}{1.07} & \secondacc{98.70}{0.16} & \secondacc{98.10}{0.32} & \acc{98.12}{0.33} & \acc{96.94}{0.51} & \acc{96.45}{0.23} \\
    \cmidrule{2-9}
    \bOnegaps{\bOneinset}{\bOneinset}
    \refinerowshade
    & \textbf{REFINE}
    & \bestacc{99.55}{0.10}
    & \bestacc{98.39}{0.40}
    & \bestacc{99.55}{0.11}
    & \bestacc{99.37}{0.09}
    & \bestacc{99.34}{0.15}
    & \bestacc{99.12}{0.08}
    & \bestacc{99.22}{0.08} \\
    \midrule

    \bOnegaps{\bOneinset}{0pt}
    \multirow{8}{*}{\textbf{VGGSound50}}
    & Standard & \acc{68.22}{1.05} & \acc{57.76}{1.22} & \acc{68.98}{0.35} & \acc{67.58}{0.50} & \acc{69.04}{0.53} & \acc{67.41}{0.92} & \acc{66.50}{0.44} \\
    \bOnegaps{0pt}{0pt}
    & CRUST & \acc{69.72}{0.52} & \acc{59.78}{2.53} & \secondacc{69.79}{0.93} & \bestacc{68.80}{1.14} & \secondacc{69.71}{0.74} & \secondacc{68.32}{1.70} & \acc{67.69}{0.54} \\
    & AUM & \secondacc{70.92}{0.15} & \acc{60.75}{2.68} & \acc{69.19}{0.54} & \acc{67.33}{0.59} & \acc{69.48}{0.66} & \acc{67.53}{0.76} & \acc{67.53}{0.52} \\
    & L2D & \acc{68.16}{0.63} & \acc{60.50}{1.55} & \acc{68.64}{0.77} & \acc{66.54}{0.46} & \acc{68.48}{0.87} & \acc{66.43}{1.07} & \acc{66.46}{0.17} \\
    & DIST & \acc{70.07}{0.46} & \secondacc{64.58}{0.78} & \acc{69.49}{0.16} & \acc{67.11}{0.69} & \acc{69.46}{0.58} & \acc{67.81}{0.84} & \secondacc{68.09}{0.21} \\
    & DIST+CT & \acc{70.07}{0.46} & \acc{64.50}{0.60} & \acc{69.49}{0.16} & \acc{67.11}{0.69} & \acc{69.46}{0.58} & \acc{67.81}{0.84} & \acc{68.07}{0.18} \\
    \bOnegaps{0pt}{\bOneinset}
    & FINE & \acc{70.23}{0.79} & \acc{57.33}{0.88} & \acc{69.24}{0.64} & \acc{67.79}{0.68} & \acc{69.35}{0.45} & \acc{67.60}{1.12} & \acc{66.92}{0.36} \\
    \cmidrule{2-9}
    \bOnegaps{\bOneinset}{\bOneinset}
    \refinerowshade
    & \textbf{REFINE}
    & \bestacc{71.29}{0.31}
    & \bestacc{65.28}{0.83}
    & \bestacc{69.87}{1.02}
    & \secondacc{68.71}{0.95}
    & \bestacc{70.72}{0.78}
    & \bestacc{69.31}{0.62}
    & \bestacc{69.20}{0.49} \\
    \midrule

    \bOnegaps{\bOneinset}{0pt}
    \multirow{8}{*}{\textbf{Rakuten France}}
    & Standard & \acc{73.24}{0.11} & \acc{62.51}{0.65} & \acc{73.54}{0.42} & \acc{70.66}{0.14} & \acc{73.74}{0.05} & \acc{71.49}{0.44} & \acc{70.86}{0.18} \\
    \bOnegaps{0pt}{0pt}
    & CRUST & \acc{73.47}{0.17} & \acc{62.38}{0.62} & \acc{73.44}{0.39} & \acc{70.73}{0.09} & \acc{73.70}{0.11} & \acc{71.49}{0.46} & \acc{70.87}{0.20} \\
    & AUM & \acc{75.11}{0.41} & \secondacc{66.20}{0.88} & \acc{74.41}{0.11} & \acc{71.99}{0.43} & \secondacc{74.04}{0.22} & \acc{71.49}{0.27} & \secondacc{72.21}{0.20} \\
    & L2D & \acc{73.20}{0.42} & \acc{64.13}{0.32} & \acc{74.19}{0.29} & \secondacc{72.36}{0.34} & \acc{73.97}{0.42} & \secondacc{72.06}{0.50} & \acc{71.65}{0.20} \\
    & DIST & \bestacc{75.15}{0.22} & \acc{61.84}{0.92} & \secondacc{74.46}{0.28} & \acc{71.64}{0.25} & \acc{73.89}{0.31} & \acc{71.79}{0.56} & \acc{71.46}{0.17} \\
    & DIST+CT & \secondacc{75.12}{0.52} & \acc{61.84}{0.91} & \bestacc{74.48}{0.30} & \acc{71.66}{0.26} & \acc{73.96}{0.34} & \acc{71.63}{0.43} & \acc{71.45}{0.16} \\
    \bOnegaps{0pt}{\bOneinset}
    & FINE & \acc{73.31}{0.59} & \acc{63.83}{0.67} & \acc{73.44}{0.39} & \acc{70.73}{0.09} & \acc{73.70}{0.11} & \acc{71.54}{0.49} & \acc{71.09}{0.17} \\
    \cmidrule{2-9}
    \bOnegaps{\bOneinset}{\bOneinset}
    \refinerowshade
    & \textbf{REFINE}
    & \acc{74.69}{0.25}
    & \bestacc{66.75}{1.67}
    & \acc{74.36}{0.15}
    & \bestacc{72.72}{0.24}
    & \bestacc{74.05}{0.34}
    & \bestacc{72.45}{0.96}
    & \bestacc{72.50}{0.21} \\
    \midrule

    \bOnegaps{\bOneinset}{0pt}
    \multirow{8}{*}{\textbf{N24News}}
    & Standard & \acc{61.10}{0.30} & \acc{49.56}{0.96} & \acc{63.37}{0.36} & \acc{62.17}{0.29} & \acc{63.64}{0.26} & \acc{61.28}{0.41} & \acc{60.19}{0.16} \\
    \bOnegaps{0pt}{0pt}
    & CRUST & \secondacc{63.68}{0.34} & \acc{55.84}{0.30} & \acc{63.68}{0.45} & \secondacc{62.95}{0.32} & \acc{63.85}{0.42} & \secondacc{62.31}{0.57} & \acc{62.05}{0.17} \\
    & AUM & \acc{63.06}{0.48} & \acc{53.01}{0.46} & \acc{63.80}{0.40} & \acc{62.57}{0.28} & \acc{63.52}{0.41} & \acc{61.74}{0.28} & \acc{61.28}{0.19} \\
    & L2D & \acc{61.05}{0.27} & \acc{52.30}{0.86} & \acc{63.49}{0.15} & \acc{62.00}{0.18} & \secondacc{63.92}{0.28} & \acc{61.64}{0.49} & \acc{60.73}{0.27} \\
    & DIST & \acc{63.09}{0.52} & \secondacc{57.83}{0.49} & \secondacc{64.01}{0.21} & \acc{62.58}{0.26} & \acc{63.82}{0.48} & \acc{62.21}{0.38} & \secondacc{62.26}{0.19} \\
    & DIST+CT & \acc{63.06}{0.49} & \acc{57.82}{0.45} & \acc{64.00}{0.20} & \acc{62.58}{0.26} & \acc{63.82}{0.49} & \acc{62.22}{0.36} & \acc{62.25}{0.18} \\
    \bOnegaps{0pt}{\bOneinset}
    & FINE & \acc{61.81}{0.40} & \acc{53.09}{0.45} & \acc{63.58}{0.39} & \acc{62.35}{0.15} & \acc{63.61}{0.28} & \acc{61.40}{0.26} & \acc{60.97}{0.16} \\
    \cmidrule{2-9}
    \bOnegaps{\bOneinset}{\bOneinset}
    \refinerowshade
    & \textbf{REFINE}
    & \bestacc{64.55}{0.53}
    & \bestacc{60.81}{0.31}
    & \bestacc{64.04}{0.28}
    & \bestacc{63.07}{0.51}
    & \bestacc{64.56}{0.41}
    & \bestacc{63.28}{0.23}
    & \bestacc{63.38}{0.27} \\
    \bottomrule
  \end{tabular}
\end{table}

\begin{table}[t]
  \caption{Classification accuracy (\%, mean$\pm$std.) of the proposed method
  and baseline methods on the BSD dataset. L2D requires a known noise rate,
  which is unavailable under real-world label noise. The best and second-best
  results are highlighted in \textbf{boldface} and \underline{underlined},
  respectively.}
  \label{tab:b2-real-world}
  \centering
  \small
  \renewcommand{\arraystretch}{1.55}
  \setlength{\aboverulesep}{0pt}
  \setlength{\belowrulesep}{0pt}
  \newcommand{\bsdres}[2]{#1\,\textpm\,#2}
  \newcommand{\bestbsdres}[2]{\textbf{#1\,\textpm\,#2}}
  \newcommand{\secondbsdres}[2]{\underline{#1\,\textpm\,#2}}
  \settowidth{\dimen0}{\textbf{Accuracy}}
  \settowidth{\dimen2}{\bsdres{00.00}{0.00}}
  \settowidth{\dimen4}{L2D}
  \settowidth{\dimen6}{\bestbsdres{00.00}{0.00}}
  \setlength{\tabcolsep}{\dimexpr
    (\textwidth-\dimen0-6\dimen2-\dimen4-\dimen6-\arrayrulewidth)/18\relax}
  \newlength{\bsdhighlightwidth}
  \setlength{\bsdhighlightwidth}{\dimexpr\dimen6+2\tabcolsep\relax}
  \newcommand{\bsdcolumnshade}{\noalign{\begingroup
    \color{gray!15}%
    \vskip1.8pt
    \hbox{\hskip\dimexpr\textwidth-\bsdhighlightwidth\relax
      \vrule width\bsdhighlightwidth height0pt depth\dimexpr
        2\ht\csname @arstrutbox\endcsname+2\dp\csname @arstrutbox\endcsname+
        \lightrulewidth-3.6pt\relax}%
    \vskip-\dimexpr
      2\ht\csname @arstrutbox\endcsname+2\dp\csname @arstrutbox\endcsname+
      \lightrulewidth-1.8pt\relax
    \endgroup}}
  \newcommand{\bsdseparator}{\smash{\rule[
    \dimexpr-\dp\csname @arstrutbox\endcsname+1.8pt\relax
  ]{\arrayrulewidth}{\dimexpr
    \ht\csname @arstrutbox\endcsname+\dp\csname @arstrutbox\endcsname-3.6pt\relax}}}
  \begin{tabular}{c!{\bsdseparator}cccccccc}
    \toprule
    \bsdcolumnshade
    \tableheadercell{\textbf{Method}}
    & \tableheadercell{Standard} & \tableheadercell{AUM} & \tableheadercell{CRUST} & \tableheadercell{L2D} & \tableheadercell{DIST} & \tableheadercell{DIST+CT} & \tableheadercell{FINE}
    & \tableheadercell{\textbf{REFINE}} \\
    \midrule
    \tableheadercell{\textbf{Accuracy}}
    & \tableheadercell{\bsdres{60.37}{0.65}}
    & \tableheadercell{\bsdres{58.90}{0.73}}
    & \tableheadercell{\bsdres{58.70}{0.71}}
    & \tableheadercell{--}
    & \tableheadercell{\bsdres{60.27}{0.74}}
    & \tableheadercell{\bsdres{60.28}{0.72}}
    & \tableheadercell{\secondbsdres{60.71}{0.98}}
    & \tableheadercell{\bestbsdres{61.03}{0.67}} \\
    \bottomrule
  \end{tabular}
\end{table}

\begin{table}[t]
  \centering
  
  \caption{Classification accuracy (\%, mean$\pm$std.) with unfrozen encoders
  on UPMC-Food101 under different noise types and ratios. The best and
  second-best results are highlighted in \textbf{boldface} and
  \underline{underlined}, respectively.}
  \label{tab:b3-end-to-end}
  \centering
  \begin{minipage}{0.66\linewidth}
  \centering
  \footnotesize
  \renewcommand{\arraystretch}{1.08}
  \newcommand{\eendres}[2]{#1\,\textpm\,#2}
  \newcommand{\besteendres}[2]{\textbf{#1\,\textpm\,#2}}
  \newcommand{\secondeendres}[2]{\underline{#1\,\textpm\,#2}}
  \settowidth{\dimen0}{DIST+CT}
  \settowidth{\dimen2}{\besteendres{00.00}{0.00}}
  \settowidth{\dimen4}{\textbf{F-IDN 50\%}}
  \setlength{\tabcolsep}{\dimexpr
    (\columnwidth-\dimen0-2\dimen2-2\dimen4-\arrayrulewidth)/10\relax}
  \begin{tabular}{c|cccc}
    \toprule
    \multicolumn{1}{c!{\tableheaderseparator}}{\tableheaderspan{\textbf{Method}}} &
    \multicolumn{4}{c}{\tableheadercell{\textbf{Noise}}} \\
    \cmidrule(l){2-5}
    & \tableheadercell{\textbf{Sym 50\%}} & \tableheadercell{\textbf{Asym 40\%}}
     & \tableheadercell{\textbf{F-IDN 50\%}} & \tableheadercell{\textbf{P-IDN 50\%}} \\
    \midrule

    Standard & \eendres{87.54}{0.20} & \eendres{73.77}{1.06}
             & \eendres{86.72}{0.15} & \eendres{86.05}{0.25} \\
    CRUST & \eendres{88.41}{0.18} & \secondeendres{81.92}{0.79}
          & \eendres{87.10}{0.21} & \eendres{86.36}{0.14} \\
    AUM & \secondeendres{88.52}{0.12} & \eendres{74.13}{1.31}
        & \eendres{87.02}{0.21} & \eendres{86.09}{0.28} \\
    L2D & \eendres{87.91}{0.17} & \eendres{80.39}{0.66}
        & \eendres{86.90}{0.21} & \eendres{86.15}{0.30} \\
    DIST & \eendres{87.53}{0.10} & \eendres{74.15}{1.29}
         & \eendres{86.73}{0.16} & \eendres{86.10}{0.27} \\
    DIST+CT & \eendres{87.50}{0.15} & \eendres{74.12}{1.30}
            & \eendres{86.73}{0.15} & \eendres{85.96}{0.42} \\
    FINE & \eendres{88.49}{0.09} & \eendres{74.13}{1.28}
         & \secondeendres{87.42}{0.12} & \secondeendres{86.76}{0.22} \\
    \midrule
    \noalign{\color{gray!15}\hrule height 2.6ex\vskip-2.6ex}

    \textbf{REFINE} & \besteendres{89.01}{0.25} & \besteendres{83.29}{0.43}
                  & \besteendres{88.12}{0.33} & \besteendres{87.56}{0.13} \\
    \bottomrule
  \end{tabular}%
  \end{minipage}
\end{table}

\textbf{Baselines.} We compare the proposed method with advanced sample-selection strategies. The baselines are commonly used in previous work. (1) \textit{Standard supervised training}:
Standard directly minimizes cross-entropy over all noisy training examples. (2) \textit{Gradient-based selection}: CRUST~\cite{mirzasoleiman2020coresets} dynamically selects representative subsets based on relationships among
example gradients within each observed class. (3) \textit{Training-dynamic-based selection}: AUM~\cite{pleiss2020identifying} selects reliable examples
using the average logit margin over training and constructed
threshold examples. L2D~\cite{jia2023learning} uses a pretrained LSTM detector
to estimate noise scores from the training trajectories of
the predicted probabilities assigned to observed labels. DIST~\cite{li2023disc}
updates per-example thresholds using an exponential moving
average of the maximum predicted class probability and retains
examples whose predicted probability for the observed label
exceeds the threshold. Moreover, DIST+CT~\cite{pan2025enhanced} adds examples that pass an
increasing confidence trend test to the set retained by DIST. (4) \textit{Representation-based selection}: FINE~\cite{kim2021fine} selects reliable examples using alignment scores
between example representations and class-wise leading eigenvectors together
with a two-component Gaussian mixture model.

All methods exploit a common multimodal classifier.
By default, retained examples contribute equally to the cross-entropy loss.
FINE uses the fused representation by default.
L2D uses the known noise rate to determine the number of retained examples
whereas DIST uses its range to set hyperparameters.
Our selection procedure does not require the noise rate.

\textbf{Implementation details.} We conducted all experiments using a GPU cluster, which includes NVIDIA TESLA A100 and NVIDIA RTX 4090 GPUs. PyTorch~\cite{paszke2019pytorch} is used. For the six synthetic-noise datasets, we use frozen DINOv2-Base~\cite{oquab2023dinov2}, XLM-R-Base~\cite{conneau2020unsupervised},
WavLM-Base+~\cite{chen2022wavlm}, and VideoMAEv2-Base~\cite{wang2023videomae} encoders. For BSD, we reused the CLAP audio and text representations provided by dataset providers. Besides, each modality representation obtained from an encoder is mapped to
256 dimensions by an independent projector comprising layer normalization,
a linear mapping, and a GELU activation.
The projected representations are concatenated and passed through a residual
fusion module with a hidden dimension of 512 to obtain a 256-dimensional
fused representation, followed by a linear classification head. We train the final classifier using AdamW~\cite{loshchilov2017decoupled}, with an initial
learning rate of 10$^\text{-4}$, weight decay of 10$^\text{-3}$, dropout of 0.3, and
a batch size of 256.
We used cosine learning rate decay updated at each optimization step.
For our method, we fix the regularization parameter $\epsilon$ to 10$^\text{-4}$,
perform 10 warm-up epochs, and then apply REFINE-based sample selection
every five epochs.
We run each method five times with different random seeds and report the mean values and standard deviations of results.

\subsection{RQ1: Overall Effectiveness under Noisy Supervision}
\label{subsec:rq1-effectiveness}

\textbf{Experiments on synthetic noise.}
Table~\ref{tab:b1-main-results} reports classification results on six datasets
under six noise settings. As can be seen, REFINE achieves the highest average accuracy on all datasets and outperforms competing methods in most settings, demonstrating consistent robustness across modality combinations and noise types. The advantage is particularly evident on MIntRec2.0, where existing selection methods do not improve over Standard on average, while REFINE remains effective despite the limited training set. Under P-IDN, REFINE also provides substantially larger gains over Standard than competing selectors, supporting the benefit of exploiting class-dependent information from multiple representation spaces.

\textbf{Experiments on realistic noise.}
Table~\ref{tab:b2-real-world} shows that REFINE reports results on BSD. REFINE achieves the highest mean accuracy. Notably, FINE and REFINE are the only selection methods that outperform Standard, suggesting that representation-based selection is particularly useful when naturally occurring label errors are not reliably characterized by training dynamics alone~\cite{kim2021fine}.

\textbf{Experiments with unfrozen encoders.}
On UPMC-Food101, we encode images and text using pretrained
ResNet-50~\cite{he2016deep} and BERT~\cite{devlin2019bert}, respectively, and jointly
fine-tune the encoders and classifier.
All methods are trained for 30 epochs using AdamW, with learning rates of
10$^\text{-6}$ and 10$^\text{-4}$ for the encoders and classification head,
respectively. Table~\ref{tab:b3-end-to-end} shows that REFINE achieves the
highest accuracy under all four settings. This result demonstrates that
our sample-selection method continues to improve multimodal classification
when the encoders are fine-tuned during training.

\subsection{RQ2: Selection Reliability and Mechanism Analysis}
\label{subsec:rq2-selection-reliability}

\textbf{Sample-selection performance.}
To evaluate how effectively different methods identify correctly labeled
examples, Table~\ref{tab:c1-selection-quality} compares their sample-selection
performance on UPMC-Food101. REFINE achieves
the highest selection F1 score under all four settings. Compared with FINE,
it improves precision in each setting while maintaining high recall, indicating that our method reduces the inclusion of mislabeled examples
while retaining correct supervision.

\begin{table}[t]
  \caption{Sample-selection precision (P), recall (R), and F1 score
  (\%, mean$\pm$std.) with respect to correctly labeled samples on UPMC-Food101. The best F1 is highlighted in \textbf{boldface}.}
  \label{tab:c1-selection-quality}
  \centering
  \scriptsize
  \renewcommand{\arraystretch}{1.10}
  \newcommand{\selres}[2]{#1\,\textpm\,#2}
  \newcommand{\bestselres}[2]{\textbf{#1\,\textpm\,#2}}
  \settowidth{\dimen0}{DIST+CT}
  \settowidth{\dimen2}{\selres{00.00}{0.00}}
  \settowidth{\dimen4}{\bestselres{00.00}{0.00}}
  \setlength{\tabcolsep}{\dimexpr
    (\textwidth-\dimen0-8\dimen2-4\dimen4-\arrayrulewidth)/26\relax}
  \begin{tabular}{c|cccccccccccc}
    \toprule
    \multicolumn{1}{c!{\tableheaderseparator}}{%
      \tableheaderspan{\textbf{Method}}} &
    \multicolumn{3}{c}{\tableheadercell{\textbf{Sym 50\%}}} &
    \multicolumn{3}{c}{\tableheadercell{\textbf{Asym 40\%}}} &
    \multicolumn{3}{c}{\tableheadercell{\textbf{F-IDN 50\%}}} &
    \multicolumn{3}{c}{\tableheadercell{\textbf{P-IDN 50\%}}} \\
    \cmidrule(lr){2-4}\cmidrule(lr){5-7}
    \cmidrule(lr){8-10}\cmidrule(l){11-13}
    & \tableheadercell{\textbf{P}} & \tableheadercell{\textbf{R}} & \tableheadercell{\textbf{F1}}
    & \tableheadercell{\textbf{P}} & \tableheadercell{\textbf{R}} & \tableheadercell{\textbf{F1}}
    & \tableheadercell{\textbf{P}} & \tableheadercell{\textbf{R}} & \tableheadercell{\textbf{F1}}
    & \tableheadercell{\textbf{P}} & \tableheadercell{\textbf{R}} & \tableheadercell{\textbf{F1}} \\
    \midrule
    CRUST
    & \selres{93.76}{0.08} & \selres{93.67}{0.09} & \selres{93.72}{0.08}
    & \selres{94.24}{0.24} & \selres{78.46}{0.20} & \selres{85.63}{0.21}
    & \selres{85.68}{0.24} & \selres{85.61}{0.23} & \selres{85.65}{0.23}
    & \selres{84.85}{0.18} & \selres{84.78}{0.17} & \selres{84.81}{0.17} \\
    AUM
    & \selres{96.40}{0.74} & \selres{90.50}{0.82} & \selres{93.35}{0.10}
    & \selres{60.62}{0.09} & \selres{91.79}{0.29} & \selres{73.01}{0.06}
    & \selres{82.36}{0.70} & \selres{93.67}{0.41} & \selres{87.65}{0.25}
    & \selres{78.39}{0.56} & \selres{93.50}{0.32} & \selres{85.28}{0.20} \\
    L2D
    & \selres{87.64}{0.25} & \selres{87.63}{0.25} & \selres{87.63}{0.25}
    & \selres{66.40}{0.09} & \selres{66.39}{0.09} & \selres{66.40}{0.09}
    & \selres{81.42}{0.09} & \selres{81.42}{0.09} & \selres{81.42}{0.09}
    & \selres{82.40}{0.13} & \selres{82.40}{0.13} & \selres{82.40}{0.13} \\
    DIST
    & \selres{94.80}{0.11} & \selres{95.98}{0.13} & \selres{95.39}{0.07}
    & \selres{93.22}{0.37} & \selres{85.92}{0.30} & \selres{89.42}{0.31}
    & \selres{81.45}{0.17} & \selres{97.04}{0.15} & \selres{88.56}{0.08}
    & \selres{80.66}{0.35} & \selres{95.89}{0.12} & \selres{87.61}{0.18} \\
    DIST+CT
    & \selres{94.80}{0.11} & \selres{95.98}{0.15} & \selres{95.39}{0.07}
    & \selres{93.16}{0.39} & \selres{85.97}{0.34} & \selres{89.42}{0.33}
    & \selres{81.49}{0.15} & \selres{97.02}{0.16} & \selres{88.58}{0.06}
    & \selres{80.74}{0.38} & \selres{95.88}{0.11} & \selres{87.66}{0.21} \\
    FINE
    & \selres{84.61}{0.17} & \selres{97.43}{0.23} & \selres{90.57}{0.11}
    & \selres{63.18}{0.26} & \selres{68.06}{0.23} & \selres{65.53}{0.20}
    & \selres{74.96}{0.86} & \selres{96.11}{0.27} & \selres{84.22}{0.50}
    & \selres{72.24}{1.44} & \selres{94.30}{0.68} & \selres{81.80}{0.81} \\
    \cmidrule{1-13}
    \noalign{\color{gray!15}\hrule height 2.6ex\vskip-2.6ex}
    \textbf{REFINE}
    & \selres{98.87}{0.06}
    & \selres{92.37}{0.18}
    & \bestselres{95.51}{0.09}
    & \selres{93.36}{0.35}
    & \selres{89.05}{0.89}
    & \bestselres{91.15}{0.54}
    & \selres{89.30}{0.14}
    & \selres{93.48}{0.38}
    & \bestselres{91.34}{0.14}
    & \selres{88.08}{0.39}
    & \selres{92.10}{0.47}
    & \bestselres{90.04}{0.09} \\
    \bottomrule
  \end{tabular}%
\end{table}

\textbf{Training dynamics of sample selection.}
We evaluate sample-selection performance at each training epoch under
asymmetric and instance-dependent noise, comparing REFINE with FINE and
DIST. Fig.~\ref{fig:c1-selection-dynamics} shows that REFINE achieves a higher
F1 score than both comparison methods at the first selection after warm-up
and maintains this advantage throughout subsequent training.

\begin{figure}[!tbp]
  \centering
  \begin{minipage}[t]{0.5717\linewidth}
    \vspace{0pt}\centering
    \includegraphics[width=\linewidth]{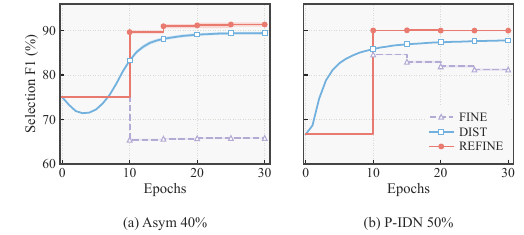}
\caption{Sample-selection F1 scores over training on UPMC-Food101 under
  Asym 40\% and P-IDN 50\%.}
  \label{fig:c1-selection-dynamics}
  \end{minipage}\hfill
  \begin{minipage}[t]{0.3983\linewidth}
    \vspace{0pt}\centering
    \includegraphics[width=\linewidth]{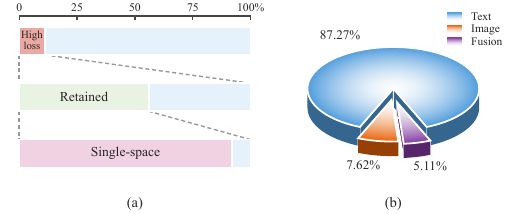}
\caption{The first selection after 10 warm-up epochs on UPMC-Food101
  under P-IDN 50\%. (a) Retention and single-space support of high-loss
  clean examples. (b) Sources of support for retained high-loss clean
  examples supported by exactly one trusted space.}
  \label{fig:c2-high-loss-clean}
  \end{minipage}
\end{figure}

\textbf{Retention of hard clean examples.}
Correctly labeled examples with high losses are commonly regarded as
hard-clean examples and are prone to being excluded by small-loss
selection~\cite{yuan2023late}. On UPMC-Food101 under P-IDN 50\%, we define the
hard-clean subset as the correctly labeled examples whose losses exceed
the median loss over the entire training set.
Fig.~\ref{fig:c2-high-loss-clean}(a) shows that this subset accounts for
11.04\% of all correctly labeled examples. REFINE retains 55.93\% of this
subset, and 91.88\% of the retained examples are supported by exactly one
trusted space. Fig.~\ref{fig:c2-high-loss-clean}(b) further shows that this
single-space support comes primarily from the text and image spaces.
This indicates that correct labels that are difficult for the current
classifier to fit can still be identified using unimodal representations.
Retaining these hard-clean examples preserves more useful supervision,
thereby improving generalization performance~\cite{yuan2023late,xia2023combating}.

\begin{table}[!t]
  \centering
  \begin{minipage}[t]{0.6035\linewidth}
    \vspace{0pt}\centering
  \caption{Component ablation results on three datasets. Except for TSS, all
  variants operate on the fused representation. The best result under
  each setting is highlighted in \textbf{boldface}.}
  \label{tab:c2-component-ablation}
  \label{tab:q3-component-ablation}% Existing cross-reference in the method section.
  \centering
  \fontsize{8}{9.6}\selectfont
  \settowidth{\dimen0}{\textbf{00.00\,\textpm\,0.00}}
  \settowidth{\dimen2}{\textbf{MIntRec2.0}\textbf{FINE}\textbf{LA}\textbf{LC}\textbf{TSS}}
  \settowidth{\dimen4}{\textbf{F-IDN 50\%}}
  \setlength{\tabcolsep}{1.5pt}
  \renewcommand{\arraystretch}{1.10}
  \newcommand{\present}{\ding{51}}
  \newcommand{\absent}{\ding{55}}
  \newcommand{\abres}[2]{#1\,\textpm\,#2}
  \newcommand{\bestabres}[2]{\textbf{#1\,\textpm\,#2}}
  \begin{adjustbox}{max width=\linewidth}
\begin{tabular}{c|cccc|cccc}
    \toprule
    \multicolumn{1}{c!{\tableheaderseparator}}{\tableheaderspan{\textbf{Dataset}}} &
    \multicolumn{4}{c!{\tableheaderseparator}}{\tableheadercell{\textbf{Method}}} &
    \multicolumn{4}{c}{\tableheadercell{\textbf{Noise}}} \\
    \cmidrule(lr){2-5}\cmidrule(l){6-9}
    & \tableheadercell{\textbf{FINE}} & \tableheadercell{\textbf{LA}} & \tableheadercell{\textbf{DE}} & \tableheadercell{\textbf{TSS}}
    & \tableheadercell{\textbf{Sym 50\%}} & \tableheadercell{\textbf{Asym 40\%}}
    & \tableheadercell{\textbf{F-IDN 50\%}} & \tableheadercell{\textbf{P-IDN 50\%}} \\
    \midrule
    \multirow{4}{*}{\textbf{N24News}}
    & \present & \absent  & \absent  & \absent
    & \abres{61.81}{0.40} & \abres{53.09}{0.45}
    & \abres{62.35}{0.15} & \abres{61.40}{0.26} \\
    & \present & \present & \absent  & \absent
    & \abres{63.15}{0.39} & \abres{56.40}{0.77}
    & \abres{62.46}{0.33} & \abres{62.72}{0.77} \\
    & \present & \present & \present & \absent
    & \abres{63.32}{0.42} & \abres{58.45}{0.80}
    & \abres{62.91}{0.31} & \abres{63.04}{0.26} \\
    & \present & \present & \present & \present
    & \bestabres{64.55}{0.53} & \bestabres{60.81}{0.31}
    & \bestabres{63.07}{0.51} & \bestabres{63.28}{0.23} \\
    \midrule
    \multirow{4}{*}{\makecell[c]{\textbf{UPMC-}\\\textbf{Food101}}}
    & \present & \absent  & \absent  & \absent
    & \abres{84.78}{0.25} & \abres{69.08}{0.56}
    & \abres{83.57}{0.15} & \abres{81.75}{0.11} \\
    & \present & \present & \absent  & \absent
    & \abres{84.85}{0.12} & \abres{82.04}{0.15}
    & \abres{83.47}{0.16} & \abres{82.64}{0.13} \\
    & \present & \present & \present & \absent
    & \abres{85.10}{0.12} & \abres{83.23}{0.35}
    & \abres{83.79}{0.52} & \abres{82.89}{0.62} \\
    & \present & \present & \present & \present
    & \bestabres{86.63}{0.22} & \bestabres{85.28}{0.34}
    & \bestabres{85.84}{0.13} & \bestabres{85.12}{0.21} \\
    \midrule
    \multirow{4}{*}{\textbf{MIntRec2.0}}
    & \present & \absent  & \absent  & \absent
    & \abres{29.73}{1.11} & \abres{25.95}{0.70}
    & \abres{33.24}{1.39} & \abres{32.14}{0.60} \\
    & \present & \present & \absent  & \absent
    & \abres{29.83}{0.74} & \abres{26.33}{1.28}
    & \abres{33.13}{1.14} & \abres{32.15}{1.30} \\
    & \present & \present & \present & \absent
    & \abres{30.41}{0.66} & \abres{28.82}{0.38}
    & \abres{33.62}{0.56} & \abres{32.85}{0.94} \\
    & \present & \present & \present & \present
    & \bestabres{31.92}{0.85} & \bestabres{29.61}{1.09}
    & \bestabres{35.70}{0.40} & \bestabres{34.94}{1.81} \\
    \bottomrule
  \end{tabular}
\end{adjustbox}

  \end{minipage}\hfill
  \begin{minipage}[t]{0.3665\linewidth}
    \vspace{7.6pt}
\centering
  \includegraphics[width=\linewidth]{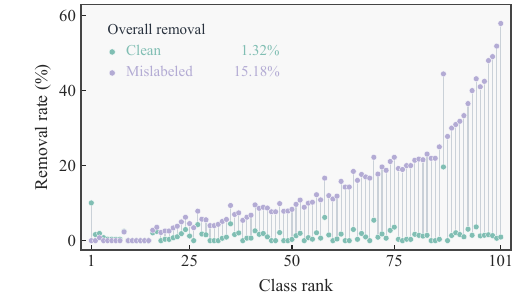}
\captionof{figure}{Removal of clean and mislabeled examples from the set selected
  by ALL through trusted space selection on UPMC-Food101 under P-IDN 50\%.
  The removal rates use the numbers of clean and mislabeled examples
  selected by ALL as their respective denominators. All 101 classes are shown,
  ordered by the difference between the removal rates for mislabeled and
  clean examples.}
  \label{fig:c3-trusted-space-filtering}

  \end{minipage}
\end{table}
\textbf{Component ablation.}
Table~\ref{tab:c2-component-ablation} progressively introduces
LA, DE, and TSS into FINE. The results show clear cumulative gains
overall, with LA and DE providing particularly large improvements
under asymmetric noise. These results support the two central design
choices: cross-class comparison for label agreement and the analysis of the target and background
to estimate discriminative eigenvectors.

\textbf{Trusted space selection.}
Table~\ref{tab:c3-modality-ablation} compares individual spaces,
their unrestricted union (abbreviated as ALL), and TSS. No single representation
space is consistently best across datasets and noise settings,
confirming that useful selection cues are class- and condition-dependent.
Moreover, TSS consistently improves over ALL, showing that exploiting
multimodal complementarity requires excluding unreliable spaces rather
than simply aggregating all available evidence. Fig.~\ref{fig:c3-trusted-space-filtering}
further shows that, relative to ALL, TSS removes 15.18\% of selected noisy
examples while discarding only 1.32\% of selected clean examples on UPMC-Food101
under P-IDN 50\%. Fig.~\ref{fig:c4-trusted-space-combinations} shows that
the most discriminative space combination also varies across classes,
further supporting class-wise trusted-space selection.

\begin{table}[!tbp]
  \centering
  \sbox{\layoutleftbox}{\begin{minipage}[t]{0.52\linewidth}
    \vspace{0pt}\caption{Representation-space comparison results on three datasets. Text,
  Image, Video, and Audio denote DE-based label agreement applied to the
  projected representation of the corresponding modality, whereas Fusion
  applies it to the fused representation. ALL independently selects examples in all available fused and
  unimodal spaces and takes their union, whereas TSS uses trusted space
  selection. The best result under each setting is highlighted in \textbf{boldface}.}
  \label{tab:c3-modality-ablation}
  \label{tab:b6-modality-ablation}% Preserve existing references.
  \centering
  \fontsize{8}{9.6}\selectfont
  \settowidth{\dimen0}{\textbf{00.00\,\textpm\,0.00}}
  \settowidth{\dimen2}{\textbf{MIntRec2.0}\textbf{Modality}}
  \settowidth{\dimen4}{\textbf{F-IDN 50\%}}
  \setlength{\tabcolsep}{3pt}
  \renewcommand{\arraystretch}{1.10}
  \newcommand{\modres}[2]{#1\,\textpm\,#2}
  \newcommand{\bestmodres}[2]{\textbf{#1\,\textpm\,#2}}
  \begin{adjustbox}{max width=\linewidth}
\begin{tabular}{c|c|cccc}
    \toprule
    \multicolumn{1}{c!{\tableheaderseparator}}{\tableheaderspan{\textbf{Dataset}}} &
    \multicolumn{1}{c!{\tableheaderseparator}}{\tableheaderspan{\textbf{Space}}} &
    \multicolumn{4}{c}{\tableheadercell{\textbf{Noise}}} \\
    \cmidrule(l){3-6}
    & & \tableheadercell{\textbf{Sym 50\%}} & \tableheadercell{\textbf{Asym 40\%}}
    & \tableheadercell{\textbf{F-IDN 50\%}} & \tableheadercell{\textbf{P-IDN 50\%}} \\
    \midrule
    \multirow{5}{*}{\textbf{N24News}}
    & Text   & \modres{63.02}{0.48} & \modres{60.21}{0.77}
             & \modres{62.17}{0.29} & \modres{61.51}{0.30} \\
    & Image  & \modres{61.08}{0.35} & \modres{50.11}{0.82}
             & \modres{62.17}{0.29} & \modres{61.51}{0.30} \\
    & Fusion & \modres{63.32}{0.42} & \modres{58.45}{0.80}
             & \modres{62.91}{0.31} & \modres{63.04}{0.26} \\
    & ALL    & \modres{63.77}{0.29} & \modres{59.48}{0.70}
             & \modres{62.51}{0.56} & \modres{62.23}{0.30} \\
    & TSS    & \bestmodres{64.55}{0.53} & \bestmodres{60.81}{0.31}
             & \bestmodres{63.07}{0.51} & \bestmodres{63.28}{0.23} \\
    \midrule
    \multirow{5}{*}{\makecell[c]{\textbf{UPMC-}\\\textbf{Food101}}}
    & Text   & \modres{84.30}{0.52} & \modres{82.53}{0.68}
             & \modres{83.79}{0.31} & \modres{83.42}{0.38} \\
    & Image  & \modres{80.63}{0.24} & \modres{74.61}{3.64}
             & \modres{81.58}{0.09} & \modres{80.56}{0.27} \\
    & Fusion & \modres{85.10}{0.12} & \modres{83.23}{0.35}
             & \modres{83.79}{0.52} & \modres{82.89}{0.62} \\
    & ALL    & \modres{86.60}{0.29} & \modres{85.17}{0.26}
             & \modres{85.66}{0.16} & \modres{84.63}{0.06} \\
    & TSS    & \bestmodres{86.63}{0.22} & \bestmodres{85.28}{0.34}
             & \bestmodres{85.84}{0.13} & \bestmodres{85.12}{0.21} \\
    \midrule
    \multirow{6}{*}{\textbf{MIntRec2.0}}
    & Text   & \modres{30.93}{1.51} & \modres{27.98}{1.40}
             & \modres{33.66}{1.42} & \modres{34.41}{2.09} \\
    & Video  & \modres{26.70}{0.83} & \modres{26.32}{0.43}
             & \modres{31.57}{0.70} & \modres{31.85}{1.30} \\
    & Audio  & \modres{26.34}{1.20} & \modres{26.65}{0.95}
             & \modres{31.57}{0.70} & \modres{31.85}{1.30} \\
    & Fusion & \modres{30.41}{0.66} & \modres{28.82}{0.38}
             & \modres{33.62}{0.56} & \modres{32.85}{0.94} \\
    & ALL    & \modres{31.12}{1.31} & \modres{29.40}{0.56}
             & \modres{35.14}{0.33} & \modres{33.24}{0.97} \\
    & TSS    & \bestmodres{31.92}{0.85} & \bestmodres{29.61}{1.09}
             & \bestmodres{35.70}{0.40} & \bestmodres{34.94}{1.81} \\
    \bottomrule
  \end{tabular}
\end{adjustbox}

  \end{minipage}}
  \sbox{\layoutrightbox}{\begin{minipage}[t]{0.46\linewidth}
    \vspace{5pt}
\centering
  \includegraphics[width=\linewidth]{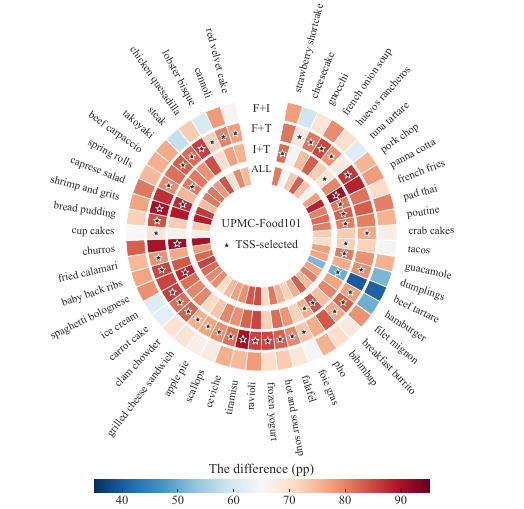}
\captionof{figure}{Results for 50 classes on UPMC-Food101 under P-IDN 50\%.
  From outermost to innermost, the rings correspond to F+I, F+T, I+T, and
  ALL. Color indicates the difference between the retention rates
  for clean and mislabeled examples, in percentage points. Stars mark the
  combinations selected by TSS. F, I, and T denote the fused, image, and
  text spaces, respectively.}
  \label{fig:c4-trusted-space-combinations}

  \end{minipage}}
  \usebox{\layoutleftbox}\hfill\usebox{\layoutrightbox}
\end{table}

\begin{figure}[!t]
  \centering
  \includegraphics[width=\textwidth]{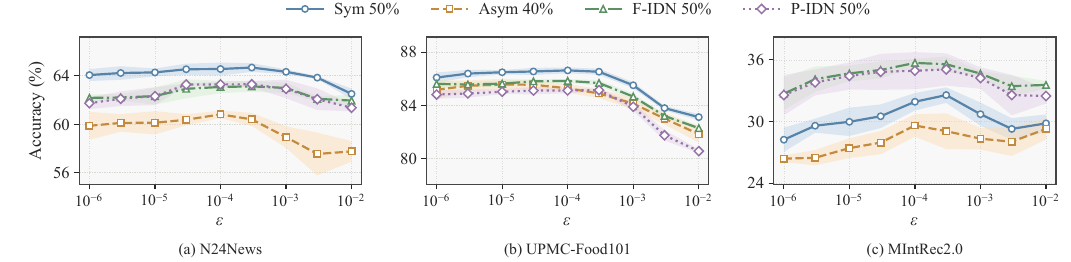}
  \caption{Sensitivity of classification accuracy (\%) to the ridge
  parameter $\epsilon$ on three datasets.}
  \label{fig:c5-rho-sensitivity}
\end{figure}

\textbf{Sensitivity analysis.}
We also analyze the sensitivity of our method to the hyperparameter $\epsilon$.
Fig.~\ref{fig:c5-rho-sensitivity} shows classification performance as $\epsilon$
varies from 10$^\text{-6}$ to 10$^\text{-2}$. In each setting, performance varies little. However, an excessively small $\epsilon$ provides
insufficient regularization of the background matrix, whereas
an excessively large $\epsilon$ increases the influence of the isotropic term,
weakening the constraint imposed by the background on class
directions.

\subsection{RQ3: Generality across LNL Training Paradigms}
\label{subsec:rq3-generality}

%\begin{wrapfigure}{r}{0.5\textwidth}
  %\centering
  %\includegraphics[width=\linewidth]{figures/d1_coteaching.pdf}
%\caption{Classification accuracy (\%) of Co-teaching and Co-teaching combined
  %with REFINE (\textit{i.e.}, R-Coteaching) on UPMC-Food101 under four noise settings.}
  %\label{fig:d1-coteaching}
%\end{wrapfigure}
\begin{figure}
    \centering
    \includegraphics[width=0.6\linewidth]{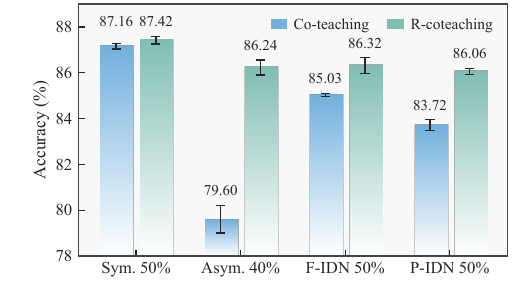}
    \caption{Classification accuracy (\%) of Co-teaching and Co-teaching combined
  with REFINE (\textit{i.e.}, R-Coteaching) on UPMC-Food101 under four noise settings.}
    \label{fig:d1-coteaching}
\end{figure}

\textbf{Extension to Co-teaching.}
We replace the small-loss selection rule in Co-teaching~\cite{han2018co}
with REFINE while retaining its cross-update mechanism: one model selects
training examples to update the other model.
Fig.~\ref{fig:d1-coteaching} shows that this replacement improves
classification accuracy under all settings.

\textbf{Extension to semi-supervised learning.}
Sample selection can also partition training data into labeled and
unlabeled subsets for semi-supervised learning. We integrate our algorithm
into FreeMatch~\cite{wang2022freematch} and DSS+~\cite{Pan_2026_CVPR}, respectively, and
compare with existing semi-supervised methods~\cite{li2020dividemix,tarvainen2017mean,verma2022interpolation,cheng2025cgmatch}.
Table~\ref{tab:d1-ssl-compatibility} shows that both combinations achieve
higher accuracy than their respective original methods, and R-DSS+ achieves
the highest results under all settings. These results demonstrate that
REFINE's selections can support subsequent semi-supervised training.

\begin{table}[!t]
  \centering
  
  \captionsetup{skip=4pt}
  \caption{Test accuracy (\%, mean$\pm$std.) of REFINE combined with
  FreeMatch (R-FreeMatch) and DSS+ (R-DSS+), compared with existing
  semi-supervised methods on UPMC-Food101 under different noise types and
  rates.
  The best and
  second-best results are highlighted in \textbf{boldface} and
  \underline{underlined}, respectively.}
  \label{tab:d1-ssl-compatibility}
  \centering
  \begin{minipage}{0.76\linewidth}
  \centering
  \footnotesize
  \renewcommand{\arraystretch}{1.08}
  \newcommand{\sslres}[2]{#1\,\textpm\,#2}
  \newcommand{\bestsslres}[2]{\textbf{#1\,\textpm\,#2}}
  \newcommand{\secondsslres}[2]{\underline{#1\,\textpm\,#2}}
  \settowidth{\dimen0}{Mean TeacherBest}
  \settowidth{\dimen2}{\bestsslres{00.00}{0.00}}
  \settowidth{\dimen4}{\textbf{F-IDN 50\%}}
  \setlength{\tabcolsep}{\dimexpr
    (\columnwidth-\dimen0-2\dimen2-2\dimen4-\arrayrulewidth-1pt)/12\relax}
  \begin{tabular}{cc|cccc}
    \toprule
    \multicolumn{2}{c!{\tableheaderseparator}}{\tableheaderspan{\textbf{Method}}} &
    \multicolumn{4}{c}{\tableheadercell{\textbf{Noise}}} \\
    \cmidrule(l){3-6}
    & & \tableheadercell{\textbf{Sym 50\%}} & \tableheadercell{\textbf{Asym 40\%}}
     & \tableheadercell{\textbf{F-IDN 50\%}} & \tableheadercell{\textbf{P-IDN 50\%}} \\
    \midrule
    \multirow{2}{*}{DivideMix}
    & Best & \sslres{86.98}{0.16} & \sslres{75.36}{0.46}
           & \sslres{86.48}{0.07} & \sslres{86.41}{0.12} \\
    & Last & \sslres{86.97}{0.17} & \sslres{72.30}{0.18}
           & \sslres{86.44}{0.06} & \sslres{86.39}{0.10} \\
    \multirow{2}{*}{Mean Teacher}
    & Best & \sslres{87.23}{0.23} & \sslres{75.26}{0.31}
           & \sslres{86.49}{0.11} & \sslres{86.36}{0.19} \\
    & Last & \sslres{87.23}{0.23} & \sslres{71.32}{0.41}
           & \sslres{86.49}{0.11} & \sslres{86.36}{0.19} \\
    \multirow{2}{*}{ICT}
    & Best & \sslres{87.09}{0.21} & \sslres{75.30}{0.37}
           & \sslres{86.41}{0.11} & \sslres{86.21}{0.13} \\
    & Last & \sslres{87.09}{0.21} & \sslres{71.47}{0.25}
           & \sslres{86.39}{0.11} & \sslres{86.21}{0.12} \\
    \multirow{2}{*}{CGMatch}
    & Best & \sslres{88.32}{0.19} & \sslres{74.19}{0.47}
           & \sslres{87.77}{0.09} & \sslres{87.58}{0.09} \\
    & Last & \sslres{88.32}{0.19} & \sslres{67.09}{0.60}
           & \sslres{87.77}{0.09} & \sslres{87.58}{0.09} \\
    \multirow{2}{*}{FreeMatch}
    & Best & \sslres{88.37}{0.21} & \sslres{73.22}{0.19}
           & \sslres{87.72}{0.13} & \sslres{87.59}{0.09} \\
    & Last & \sslres{88.37}{0.21} & \sslres{66.66}{0.68}
           & \sslres{87.71}{0.13} & \sslres{87.58}{0.09} \\
    \multirow{2}{*}{DSS+}
    & Best & \secondsslres{89.32}{0.09} & \secondsslres{87.75}{0.86}
           & \secondsslres{89.12}{0.07} & \secondsslres{89.15}{0.04} \\
    & Last & \secondsslres{89.32}{0.09} & \secondsslres{87.70}{0.91}
           & \secondsslres{89.09}{0.07} & \secondsslres{89.13}{0.04} \\
    \midrule
    \multirow{2}{*}{\textbf{R-FreeMatch}}
    & Best & \sslres{89.07}{0.10} & \sslres{87.08}{0.95}
           & \sslres{88.25}{0.08} & \sslres{87.93}{0.06} \\
    & Last & \sslres{89.06}{0.10} & \sslres{86.97}{0.90}
           & \sslres{88.23}{0.08} & \sslres{87.91}{0.09} \\
    \multirow{2}{*}{\textbf{R-DSS+}}
    & Best & \bestsslres{89.69}{0.06} & \bestsslres{89.53}{0.08}
           & \bestsslres{89.30}{0.11} & \bestsslres{89.27}{0.13} \\
    & Last & \bestsslres{89.68}{0.05} & \bestsslres{89.50}{0.09}
           & \bestsslres{89.29}{0.11} & \bestsslres{89.24}{0.14} \\
    \bottomrule
  \end{tabular}%
  \end{minipage}
\end{table}

\textbf{Combination with noise-tolerance loss functions.}
Noise-robust losses mitigate the influence of incorrect labels on training by
adjusting the optimization objective. We combine REFINE with different
noise-robust losses~\cite{zhang2018generalized,wang2019symmetric,wang2025joint}
and use CE as a reference. Fig.~\ref{fig:d2-robust-losses} shows
that incorporating REFINE increases the accuracy for every training
objective, demonstrating that our selection method can accommodate
different training objectives. Further comparison shows that combining
REFINE with GCE or JAL-CE yields slightly
higher accuracy than REFINE+CE under symmetric noise, but lower accuracy
under asymmetric and both types of instance-dependent noise.
We believe this is because the training loss alters the representations
used for selection, which in turn affects the degree of support for correct and
incorrect labels and changes how robust optimization and sample selection
work together. How training objectives affect representation
discriminability and sample-selection quality under different noise
conditions warrants further investigation.

\begin{figure}[!tbp]
  \centering
  \includegraphics[width=0.22\linewidth,trim=90bp 220.740158bp 53.058533bp 0bp,clip]{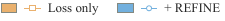}\\[2pt]
  \includegraphics[height=0.212\linewidth,trim=0bp 110.740158bp 121.058533bp 10bp,clip]{figures/d2_robust_losses.pdf}\hfill
  \includegraphics[height=0.212\linewidth,trim=130bp 110.740158bp 0bp 10bp,clip]{figures/d2_robust_losses.pdf}\hfill
  \includegraphics[height=0.212\linewidth,trim=0bp 0.700158bp 121.058533bp 120.04bp,clip]{figures/d2_robust_losses.pdf}\hfill
  \includegraphics[height=0.212\linewidth,trim=130bp 0.700158bp 0bp 120.04bp,clip]{figures/d2_robust_losses.pdf}
\caption{Classification accuracy (\%) of different training losses with and
  without REFINE on UPMC-Food101 under four noise settings.}
  \label{fig:d2-robust-losses}
\end{figure}

\FloatBarrier
\section{Conclusion}
\label{sec:conclusion}

This work studies sample selection for multimodal learning under noisy supervision,
focusing on which representation spaces to trust for identifying reliable training data.
We show theoretically that fused and unimodal representations can differ in their clean-noisy
separation ability across classes. Based on this insight, we develop REFINE,
which constructs discriminative eigenvectors through analysis of the target and background,
selects trusted representation spaces in a class-wise manner,
and identifies reliable examples through label agreement.
Extensive experiments across diverse datasets, noise settings,
and training paradigms demonstrate consistent improvements
in both sample-selection quality and classification performance.
Overall, our results highlight the importance of explicitly determining what to trust and where to trust it in multimodal learning with noisy labels. Future work will extend this perspective toward more general forms of imperfect multimodal supervision, including partially corrupted, missing, and inconsistent modality-level signals. More broadly, we aim to develop a unified framework for trustworthy multimodal learning that jointly determines which examples, representation spaces, and supervision signals to be trusted under imperfect supervision.

\clearpage
\onecolumn

\renewcommand{\cftsecfont}{\normalsize}
\renewcommand{\cftsubsecfont}{\normalsize}
\renewcommand{\cftbeforesecskip}{12pt}
\renewcommand{\cftbeforesubsecskip}{12pt}

\renewcommand{\contentsname}{Appendix}
\addtocontents{toc}{\protect\setcounter{tocdepth}{2}}

  \appendix
  \tableofcontents

    \clearpage

\renewcommand{\thesubsection}{\thesection.\arabic{subsection}}

\section{Theoretical Analysis}
\label{app:theory}

\subsection{Proof of Theorem~\ref{thm:fine-alignment-score-difference}}
\label{app:proof-fine-gap}

\textit{Recall.}
Under the model in Section~\ref{subsec:problem-formulation}, the
alignment-score difference produced by FINE in space $r\in\mathcal R$ is
\begin{equation}
\Delta_{\mathrm{FINE}}^r
:=\Delta^r(\bm{u}_{\mathrm{FINE}}^r)
=\frac{(1-\tau)\sin^2\theta^r}
{\sqrt{(1-\tau)^2+4\tau\cos^2\theta^r}},
\label{eq:app-fine-gap}
\end{equation}
where $\tau$ is the ratio of mislabeled to clean examples, \textit{i.e.},
$\frac{\mathbb P(y=-1\mid\widetilde y=+1)}
{\mathbb P(y=+1\mid\widetilde y=+1)}$.
For any fixed $\tau\in[0,1)$, $\Delta_{\mathrm{FINE}}^r$ is
strictly increasing in $\theta^r\in[0,\pi/2]$. For any fixed
$\theta^r\in(0,\pi/2)$, it is strictly decreasing in $\tau\in[0,1)$.
In particular, all representation spaces share the same $\tau$, so for
any $r_1,r_2\in\mathcal R$,
\begin{equation}
\Delta_{\mathrm{FINE}}^{r_1}>\Delta_{\mathrm{FINE}}^{r_2}
\quad\Longleftrightarrow\quad
\theta^{r_1}>\theta^{r_2}.
\label{eq:app-fine-space-ordering}
\end{equation}

\begin{proof}
Within the observed class $\widetilde y=+1$, the probabilities of
correctly labeled and mislabeled examples are $1/(1+\tau)$ and
$\tau/(1+\tau)$, respectively. The class matrix in space $r$
is therefore
\begin{equation}
\begin{aligned}
\bm{A}_+^r
&=\frac{
\mathbb E[\bm{z}^r(\bm{z}^r)^\top\mid y=+1]
+\tau\mathbb E[\bm{z}^r(\bm{z}^r)^\top\mid y=-1]}{1+\tau}\\
&=\sigma^2 \bm{I}+
\frac{\bm{v}^r(\bm{v}^r)^\top+\tau \bm{w}^r(\bm{w}^r)^\top}{1+\tau}.
\end{aligned}
\label{eq:app-target-second-moment}
\end{equation}

Its leading eigenvector is the same as that of
$\bm{v}^r(\bm{v}^r)^\top+\tau \bm{w}^r(\bm{w}^r)^\top$. Choose the signs of the directions
so that their inner product is nonnegative. In orthonormal coordinates
on the plane spanned by the mean directions, $\bm{w}^r$ and $\bm{v}^r$ have
coordinates $(1,0)^\top$ and $(\cos\theta^r,\sin\theta^r)^\top$,
respectively. The matrix restricted to this plane is
\begin{equation}
\begin{bmatrix}
\cos^2\theta^r+\tau&\sin\theta^r\cos\theta^r\\
\sin\theta^r\cos\theta^r&\sin^2\theta^r
\end{bmatrix}.
\label{eq:app-fine-planar-matrix}
\end{equation}

The leading eigenvector $\bm{u}_{\mathrm{FINE}}^r$ has coordinates
$(\cos\varphi^r,\sin\varphi^r)^\top$ on this plane. Its Rayleigh
quotient can be written as
\begin{equation}
\begin{aligned}
(\bm{u}_{\mathrm{FINE}}^r)^\top
\bigl[\bm{v}^r(\bm{v}^r)^\top+\tau \bm{w}^r(\bm{w}^r)^\top\bigr]
\bm{u}_{\mathrm{FINE}}^r&=\cos^2(\theta^r-\varphi^r)+\tau\cos^2\varphi^r\\
&=\frac{1+\tau}{2}
+\frac{(\tau+\cos2\theta^r)\cos2\varphi^r
+\sin2\theta^r\sin2\varphi^r}{2}.
\end{aligned}
\label{eq:app-fine-rayleigh}
\end{equation}

Since $(\cos2\varphi^r,\sin2\varphi^r)^\top$ is a unit vector,
the maximum is attained when it points in the same direction as
$(\tau+\cos2\theta^r,\sin2\theta^r)^\top$. Moreover,
\begin{equation}
(\tau+\cos2\theta^r)^2+\sin^2(2\theta^r)
=(1-\tau)^2+4\tau\cos^2\theta^r,
\label{eq:app-fine-double-angle-normalizer}
\end{equation}
and hence
\begin{equation}
\cos2\varphi^r=\frac{\tau+\cos2\theta^r}{\sqrt{(1-\tau)^2+4\tau\cos^2\theta^r}},
\qquad
\sin2\varphi^r=\frac{\sin2\theta^r}{\sqrt{(1-\tau)^2+4\tau\cos^2\theta^r}}.
\label{eq:app-fine-eigenvector-angle}
\end{equation}

By the definition of the alignment-score difference,
\begin{equation}
\begin{aligned}
\Delta_{\mathrm{FINE}}^r
&=\bigl[(\bm{u}_{\mathrm{FINE}}^r)^\top \bm{v}^r\bigr]^2
-\bigl[(\bm{u}_{\mathrm{FINE}}^r)^\top \bm{w}^r\bigr]^2\\
&=\frac{(1-\tau)\sin^2\theta^r}{\sqrt{(1-\tau)^2+4\tau\cos^2\theta^r}}.
\end{aligned}
\label{eq:app-fine-gap-derivation}
\end{equation}

Differentiating the closed-form expression gives
\begin{equation}
\begin{aligned}
\frac{\partial\Delta_{\mathrm{FINE}}^r}{\partial\theta^r}
&=\frac{2(1-\tau)(1+\tau^2+2\tau\cos^2\theta^r)
\sin\theta^r\cos\theta^r}{[(1-\tau)^2+4\tau\cos^2\theta^r]^{3/2}}>0,\\
\frac{\partial\Delta_{\mathrm{FINE}}^r}{\partial\tau}
&=-\frac{2(1+\tau)\sin^2\theta^r\cos^2\theta^r}{[(1-\tau)^2+4\tau\cos^2\theta^r]^{3/2}}<0,
\qquad 0<\theta^r<\pi/2.
\end{aligned}
\label{eq:app-fine-gap-derivatives}
\end{equation}
Continuity of the closed-form expression extends the angular
monotonicity to the entire interval. All spaces use the same observed
class and therefore share $\tau$, which establishes the ordering
across spaces.
\end{proof}

\subsection{Proof of Theorem~\ref{thm:de-gap}}
\label{app:proof-de-gap}

\textit{Recall.}
Under the model in Section~\ref{subsec:problem-formulation}, if
$\epsilon=0$, then
\begin{equation}
\Delta_{\mathrm{DE}}^r
:=\Delta^r(\bm{u}_{\mathrm{DE}}^r)
=\sin\theta^r
\sqrt{1-\frac{\cos^2\theta^r}{(1+2\sigma^2)^2}}
\geq\Delta_{\mathrm{FINE}}^r.
\label{eq:app-de-gap}
\end{equation}

The inequality is strict when $0<\theta^r<\pi/2$, and
$\Delta_{\mathrm{DE}}^r-\Delta_{\mathrm{FINE}}^r$ strictly increases
with $\tau$. For fixed $\sigma>0$, $\Delta_{\mathrm{DE}}^r$ is
strictly increasing in $\theta^r\in[0,\pi/2]$.

\begin{proof}
Let $\tau_-\in[0,1)$ be the ratio of mislabeled to correctly labeled
examples in the observed class $\widetilde y=-1$. In space $r$,
choose an orthonormal basis along the angle bisectors of the mean
directions such that
\begin{equation}
\bm{v}^r=\cos\frac{\theta^r}{2}\,\bm{e}_1^r
+\sin\frac{\theta^r}{2}\,\bm{e}_2^r,
\qquad \bm{w}^r=\cos\frac{\theta^r}{2}\,\bm{e}_1^r
-\sin\frac{\theta^r}{2}\,\bm{e}_2^r.
\label{eq:app-mean-bisector-coordinates}
\end{equation}

The second-moment matrices of the two observed classes,
restricted to this plane, are
\begin{equation}
\begin{aligned}
A_+^r&=
\begin{bmatrix}
\sigma^2+\cos^2(\theta^r/2)&
\dfrac{1-\tau}{2(1+\tau)}\sin\theta^r\\
\dfrac{1-\tau}{2(1+\tau)}\sin\theta^r&
\sigma^2+\sin^2(\theta^r/2)
\end{bmatrix},\\
A_-^r&=
\begin{bmatrix}
\sigma^2+\cos^2(\theta^r/2)&
-\dfrac{1-\tau_-}{2(1+\tau_-)}\sin\theta^r\\
-\dfrac{1-\tau_-}{2(1+\tau_-)}\sin\theta^r&
\sigma^2+\sin^2(\theta^r/2)
\end{bmatrix}.
\end{aligned}
\label{eq:app-observed-planar-moments}
\end{equation}

DE is obtained by solving for the eigenvector associated with
the largest generalized eigenvalue,
$\bm{A}_+^r \bm{u}_{\mathrm{DE}}^r=\lambda_+^r \bm{A}_-^r \bm{u}_{\mathrm{DE}}^r$.
Expanding this equation along the two basis directions gives
\begin{equation}
\begin{aligned}
&(1-\lambda_+^r)
[\sigma^2+\cos^2(\theta^r/2)]
[(\bm{e}_1^r)^\top \bm{u}_{\mathrm{DE}}^r]+\left(\frac{1-\tau}{1+\tau}
+\lambda_+^r\frac{1-\tau_-}{1+\tau_-}\right)
\frac{\sin\theta^r}{2}
[(\bm{e}_2^r)^\top \bm{u}_{\mathrm{DE}}^r]=0,\\
&\left(\frac{1-\tau}{1+\tau}
+\lambda_+^r\frac{1-\tau_-}{1+\tau_-}\right)
\frac{\sin\theta^r}{2}
[(\bm{e}_1^r)^\top \bm{u}_{\mathrm{DE}}^r]+(1-\lambda_+^r)
[\sigma^2+\sin^2(\theta^r/2)]
[(\bm{e}_2^r)^\top \bm{u}_{\mathrm{DE}}^r]=0.
\end{aligned}
\label{eq:app-de-coordinate-equations}
\end{equation}

Multiply the first equation by
$[\sigma^2+\sin^2(\theta^r/2)][(\bm{e}_2^r)^\top \bm{u}_{\mathrm{DE}}^r]$
and the second by
$[\sigma^2+\cos^2(\theta^r/2)][(\bm{e}_1^r)^\top \bm{u}_{\mathrm{DE}}^r]$,
and then subtract. Since $\bm{A}_\pm^r\succ0$, we have
$\lambda_+^r>0$. Together with $\tau,\tau_-<1$, this yields

\begin{equation}
\bigl[\sigma^2+\sin^2(\theta^r/2)\bigr]
\bigl[(\bm{e}_2^r)^\top \bm{u}_{\mathrm{DE}}^r\bigr]^2
=
\bigl[\sigma^2+\cos^2(\theta^r/2)\bigr]
\bigl[(\bm{e}_1^r)^\top \bm{u}_{\mathrm{DE}}^r\bigr]^2.
\label{eq:app-de-coordinate-relation}
\end{equation}

Choose the maximizing direction whose coordinates have the same
sign. Normalizing it to unit length gives

\begin{equation}
\bm{u}_{\mathrm{DE}}^r
=\frac{
\sqrt{\sigma^2+\sin^2(\theta^r/2)}\,\bm{e}_1^r
+\sqrt{\sigma^2+\cos^2(\theta^r/2)}\,\bm{e}_2^r}
{\sqrt{1+2\sigma^2}}.
\label{eq:app-de-direction}
\end{equation}

Substituting into the definition of the alignment-score
difference gives

\begin{equation}
\begin{aligned}
\Delta_{\mathrm{DE}}^r
&=\frac{2\sin\theta^r
\sqrt{\bigl[\sigma^2+\sin^2(\theta^r/2)\bigr]
\bigl[\sigma^2+\cos^2(\theta^r/2)\bigr]}}
{1+2\sigma^2}\\
&=\sin\theta^r
\sqrt{1-\frac{\cos^2\theta^r}{(1+2\sigma^2)^2}}.
\end{aligned}
\label{eq:app-de-gap-derivation}
\end{equation}

Moreover,

\begin{equation}
(\Delta_{\mathrm{DE}}^r)^2-\sin^4\theta^r
=\sin^2\theta^r\cos^2\theta^r
\left[1-(1+2\sigma^2)^{-2}\right]\geq0,
\label{eq:app-de-squared-gap-comparison}
\end{equation}

which, together with (\ref{eq:app-fine-gap}), implies
\nopagebreak[4]

\begin{equation}
\Delta_{\mathrm{DE}}^r\geq\sin^2\theta^r
\geq\frac{(1-\tau)\sin^2\theta^r}
{\sqrt{(1-\tau)^2+4\tau\cos^2\theta^r}}
=\Delta_{\mathrm{FINE}}^r.
\label{eq:app-de-fine-comparison}
\end{equation}

Finally,

\begin{equation}
\frac{\partial(\Delta_{\mathrm{DE}}^r)^2}{\partial\theta^r}
=2\sin\theta^r\cos\theta^r
\left[1-\frac{\cos2\theta^r}{(1+2\sigma^2)^2}\right]>0,
\qquad 0<\theta^r<\pi/2.
\label{eq:app-de-gap-derivative}
\end{equation}

Thus, $\Delta_{\mathrm{DE}}^r$ strictly increases with the class
angle. Together with (\ref{eq:app-fine-gap-derivatives}), this gives
$\partial_\tau(\Delta_{\mathrm{DE}}^r-\Delta_{\mathrm{FINE}}^r)
=-\partial_\tau\Delta_{\mathrm{FINE}}^r>0$.
\end{proof}

\subsection{Proof of Proposition~\ref{prop:background-weight-ordering}}
\label{app:proof-background-weights}

\textit{Recall.}
Fix target class $k$ and space $r$. Suppose that each non-target observed
class satisfies
$\bm A_j^r=\bm v_j^r(\bm v_j^r)^\top+\sigma^2\bm I$, where
$\lVert\bm v_j^r\rVert_2=1$ and $\sigma>0$. Let
$\phi_{kj}^r=\arccos\bigl|(\bm u_{k,\mathrm{FINE}}^r)^\top
\bm v_j^r\bigr|$. For any $j,\ell\neq k$,
\begin{equation}
    d_{kj}^r=\sigma^2+\cos^2\phi_{kj}^r,
    \qquad \phi_{kj}^r<\phi_{k\ell}^r
    \iff\gamma_{kj}^r>\gamma_{k\ell}^r.
    \label{eq:background-weight-ordering-appendix}
\end{equation}

\begin{proof}
Since $\lVert\bm{u}_{k,\mathrm{FINE}}^r\rVert_2=1$,

\begin{equation}
\begin{aligned}
d_{kj}^r
&=(\bm{u}_{k,\mathrm{FINE}}^r)^\top \bm{A}_j^r \bm{u}_{k,\mathrm{FINE}}^r\\
&=\sigma^2+\bigl[(\bm{u}_{k,\mathrm{FINE}}^r)^\top \bm{v}_j^r\bigr]^2\\
&=\sigma^2+\cos^2\phi_{kj}^r.
\end{aligned}
\label{eq:app-background-score-expansion}
\end{equation}

$\cos^2\phi_{kj}^r$ strictly decreases with
$\phi_{kj}^r\in[0,\pi/2]$, and all weights share the positive
denominator $\sum_{j\neq k}d_{kj}^r\geq(K-1)\sigma^2>0$.
Therefore, $\phi_{kj}^r<\phi_{k\ell}^r$ is equivalent to
$d_{kj}^r>d_{k\ell}^r$, and hence to
$\gamma_{kj}^r>\gamma_{k\ell}^r$.
\end{proof}

\subsection{Proof of Theorem~\ref{thm:two-dimensional-de-ordering}}
\label{app:proof-space-ordering}

\textit{Recall.}
Under the model in Section~\ref{subsec:problem-formulation}, if
$\epsilon=0$, then
\begin{equation}
|a_k^r|>|a_k^s|
\quad\Longleftrightarrow\quad
\Delta_{\mathrm{DE}}^r>\Delta_{\mathrm{DE}}^s
\quad\Longleftrightarrow\quad
\theta^r>\theta^s.
\label{eq:app-space-ordering}
\end{equation}

\begin{proof}
Without loss of generality, assume $\theta^r>\theta^s$. For
$q\in\{r,s\}$, define the projection onto the DE direction used in the
main text and its second moments under the two true classes as
\begin{equation}
p_k^q=(\bm{u}_{k,\mathrm{DE}}^q)^\top \bm{z}^q,
\qquad
P_k^q=\mathbb E[(p_k^q)^2\mid y=+1],
\quad Q_k^q=\mathbb E[(p_k^q)^2\mid y=-1].
\label{eq:app-projection-moment-definitions}
\end{equation}

Using (\ref{eq:app-de-direction}) and $P_k^q-Q_k^q=\Delta_{\mathrm{DE}}^q$, we obtain
\begin{equation}
\begin{aligned}
P_k^q+Q_k^q
&=2\sigma^2+
[(\bm{u}_{k,\mathrm{DE}}^q)^\top \bm{v}^q]^2+
[(\bm{u}_{k,\mathrm{DE}}^q)^\top \bm{w}^q]^2\\
&=\frac{\sin^2\theta^q+4\sigma^2(1+\sigma^2)}{1+2\sigma^2},\\
Q_k^q
&=\frac{\sin^2\theta^q+4\sigma^2(1+\sigma^2)}{2(1+2\sigma^2)}
-\frac{\Delta_{\mathrm{DE}}^q}{2}.
\end{aligned}
\label{eq:app-class-projection-moments}
\end{equation}

The first expression strictly increases with $\sin^2\theta^q$.
By (\ref{eq:app-de-gap}), the second satisfies
\begin{equation}
\frac{\mathrm dQ_k^q}{\mathrm d(\sin^2\theta^q)}
=\frac{1}{2(1+2\sigma^2)}
\left[
1-\frac{\sin^2\theta^q+2\sigma^2(1+\sigma^2)}
{\sqrt{\sin^2\theta^q[\sin^2\theta^q+4\sigma^2(1+\sigma^2)]}}
\right]<0,
\label{eq:app-noisy-projection-moment-derivative}
\end{equation}

Therefore, $P_k^r+Q_k^r>P_k^s+Q_k^s$ and $Q_k^r<Q_k^s$.
For the two-dimensional projection
$\bm{p}_k^{r,s}=(p_k^r,p_k^s)^\top$ defined in the main text, let
\begin{equation}
\bm{P}_k^{r,s}
=\mathbb E[\bm{p}_k^{r,s}(\bm{p}_k^{r,s})^\top\mid y=+1],
\qquad
 \bm{Q}_k^{r,s}
=\mathbb E[\bm{p}_k^{r,s}(\bm{p}_k^{r,s})^\top\mid y=-1].
\label{eq:app-paired-true-class-moments}
\end{equation}

The target and background matrices satisfy
\begin{equation}
\bm{C}_{k,k}^{r,s}
=\frac{ \bm{P}_k^{r,s}+\tau \bm{Q}_k^{r,s}}{1+\tau},
\qquad
 \bm{D}_k^{r,s}
=\frac{\tau_- \bm{P}_k^{r,s}+ \bm{Q}_k^{r,s}}{1+\tau_-}.
\label{eq:app-paired-mixture-moments}
\end{equation}

Let $\nu_k^q=\mathbb E[p_k^q\mid y=-1]$. Then
$Q_k^q=\sigma^2+(\nu_k^q)^2$. The Cauchy--Schwarz inequality gives
$\bigl|[\bm{Q}_k^{r,s}]_{12}\bigr|\leq\sigma^2+|\nu_k^r\nu_k^s|$.
Consequently,
\begin{equation}
\begin{aligned}
\det Q_k^{r,s}
&\geq[\sigma^2+(\nu_k^r)^2][\sigma^2+(\nu_k^s)^2]
-(\sigma^2+|\nu_k^r\nu_k^s|)^2\\
&=\sigma^2(|\nu_k^r|-|\nu_k^s|)^2>0.
\end{aligned}
\label{eq:app-paired-background-determinant}
\end{equation}

The strict inequality follows from $Q_k^r<Q_k^s$. Therefore,
$\bm{Q}_k^{r,s}\succ0$ and
$\bm{D}_k^{r,s}\succeq \bm{Q}_k^{r,s}/(1+\tau_-)\succ0$.

Let $\lambda_k^{r,s}\geq0$ be the largest generalized
eigenvalue of $(\bm{C}_{k,k}^{r,s},\bm{D}_k^{r,s})$, with the corresponding
Euclidean unit eigenvector $\bm{a}_k^{r,s}=(a_k^r,a_k^s)^\top$ defined in
the main text. From (\ref{eq:app-paired-mixture-moments}),
\begin{equation}
(1+\tau_-) \bm{D}_k^{r,s}
-\tau_-(1+\tau) \bm{C}_{k,k}^{r,s}
=(1-\tau\tau_-) \bm{Q}_k^{r,s}\succ0.
\label{eq:app-paired-positive-combination}
\end{equation}

Taking the quadratic form along $\bm{a}_k^{r,s}$ and using
$\bm{C}_{k,k}^{r,s}\bm{a}_k^{r,s}=\lambda_k^{r,s}\bm{D}_k^{r,s}\bm{a}_k^{r,s}$
yields
\begin{equation}
[(1+\tau_-)-\tau_-(1+\tau)\lambda_k^{r,s}]
(\bm{a}_k^{r,s})^\top \bm{D}_k^{r,s} \bm{a}_k^{r,s}\\
=(1-\tau\tau_-)
(\bm{a}_k^{r,s})^\top \bm{Q}_k^{r,s} \bm{a}_k^{r,s}>0.
\label{eq:app-paired-eigenvalue-quadratic-form}
\end{equation}

Since $\bm{D}_k^{r,s}\succ0$, dividing by the positive quadratic
form and by $(1+\tau)(1+\tau_-)$ gives
\begin{equation}
\frac1{1+\tau}-\frac{\tau_-\lambda_k^{r,s}}{1+\tau_-}>0.
\label{eq:app-paired-eigenvalue-inequality}
\end{equation}

Define the residual matrix of the generalized eigenvalue
problem as $\bm{G}_k^{r,s}=\lambda_k^{r,s}\bm{D}_k^{r,s}-\bm{C}_{k,k}^{r,s}$.
The definition of the largest generalized eigenvalue gives
$\bm{G}_k^{r,s}\succeq0$ and $\bm{G}_k^{r,s}\bm{a}_k^{r,s}=\bm{0}$.
Substituting (\ref{eq:app-paired-mixture-moments}) into the difference between its diagonal entries
yields
\begin{equation}
\begin{aligned}
\bigl[\bm{G}_k^{r,s}\bigr]_{22}-\bigl[\bm{G}_k^{r,s}\bigr]_{11}
&=\left(\frac1{1+\tau}-\frac{\tau_-\lambda_k^{r,s}}{1+\tau_-}\right)
\bigl[(P_k^r+Q_k^r)-(P_k^s+Q_k^s)\bigr]\\
&\quad-\left(\frac{1-\tau}{1+\tau}
+\frac{(1-\tau_-)\lambda_k^{r,s}}{1+\tau_-}\right)
(Q_k^r-Q_k^s)>0.
\end{aligned}
\label{eq:app-residual-diagonal-difference}
\end{equation}

Thus, $\bm{G}_k^{r,s}$ is nonzero and has rank one, and its null
space is spanned by $\bm{a}_k^{r,s}$. It follows that
\begin{equation}
\bm{G}_k^{r,s}
=\operatorname{tr}(\bm{G}_k^{r,s})
\bigl[\bm{I}- \bm{a}_k^{r,s}(\bm{a}_k^{r,s})^\top\bigr],
\label{eq:app-residual-projector}
\end{equation}

and therefore
\begin{equation}
(a_k^r)^2-(a_k^s)^2
=\frac{\bigl[\bm{G}_k^{r,s}\bigr]_{22}-\bigl[\bm{G}_k^{r,s}\bigr]_{11}}
{\operatorname{tr} \bm{G}_k^{r,s}}>0.
\label{eq:app-coefficient-square-difference}
\end{equation}

Interchanging the two spaces gives the reverse ordering.
Combining this result with the angular monotonicity in
Theorem~\ref{thm:de-gap} establishes (\ref{eq:app-space-ordering}).
\end{proof}

\subsection{Lower Bounds on Recall and Precision}
\label{app:recall-precision}
Here, we derive lower bounds on the recall and
precision of the REFINE detector in each representation space and
relate these bounds to the alignment-score difference.

\begin{theorem}
\label{thm:recall-precision-bounds}
Under the binary model in the main text, let $\epsilon=0$. For the
LA selection rule in space $r$, the recall and precision within the
target observed class $\widetilde y=+1$ satisfy
\begin{equation}
\begin{aligned}
\operatorname{Recall}^r
&\geq\frac{(\Delta_{\mathrm{DE}}^r)^2}
{(\Delta_{\mathrm{DE}}^r)^2
+4\sigma^2(1+\sigma^2)
\left[1-\frac{\cos^2\theta^r}{(1+2\sigma^2)^2}\right]}\\
&=\frac{\sin^2\theta^r}{\sin^2\theta^r+4\sigma^2(1+\sigma^2)},\\
\operatorname{Precision}^r
&\geq\frac{(\Delta_{\mathrm{DE}}^r)^2}
{(\Delta_{\mathrm{DE}}^r)^2
+4\tau\sigma^2(1+\sigma^2)
\left[1-\frac{\cos^2\theta^r}{(1+2\sigma^2)^2}\right]}\\
&=\frac{\sin^2\theta^r}{\sin^2\theta^r+4\tau\sigma^2(1+\sigma^2)}.
\end{aligned}
\label{eq:app-recall-precision-bounds}
\end{equation}

For fixed $\sigma>0$ and $\tau>0$, both lower bounds strictly
increase with $\Delta_{\mathrm{DE}}^r$.
\end{theorem}

\begin{proof}
In space $r$, $\bm{u}_{+,\mathrm{DE}}^r=\bm{u}_{\mathrm{DE}}^r$.
Interchanging the target and background reverses the sign of the
second basis coordinate in (\ref{eq:app-de-direction}), yielding $\bm{u}_{-,\mathrm{DE}}^r$.
The label agreement (LA) retention event is therefore equivalent to $M^r>0$, where
\begin{equation}
\begin{aligned}
M^r
&=[(\bm{u}_{+,\mathrm{DE}}^r)^\top \bm{z}^r]^2
-[(\bm{u}_{-,\mathrm{DE}}^r)^\top \bm{z}^r]^2\\
&=2\sqrt{1-\frac{\cos^2\theta^r}{(1+2\sigma^2)^2}}\,
[(\bm{e}_1^r)^\top \bm{z}^r][(\bm{e}_2^r)^\top \bm{z}^r].
\end{aligned}
\label{eq:app-la-margin}
\end{equation}

Together with (\ref{eq:app-de-gap-derivation}), this gives

\begin{equation}
\begin{aligned}
\mathbb E[M^r\mid y=+1]&=\Delta_{\mathrm{DE}}^r,
\qquad
\mathbb E[M^r\mid y=-1]=-\Delta_{\mathrm{DE}}^r,\\
V^r:=\operatorname{Var}(M^r\mid y=\pm1)
&=4\sigma^2(1+\sigma^2)
\left[1-\frac{\cos^2\theta^r}{(1+2\sigma^2)^2}\right].
\end{aligned}
\label{eq:app-la-margin-moments}
\end{equation}

Let $q_{\mathrm{err}}^r=\mathbb P(M^r>0\mid y=-1)$ denote the
probability that a mislabeled example is retained. Rejecting a positive-class
example requires its margin to deviate downward from its mean by at
least $\Delta_{\mathrm{DE}}^r$. Cantelli's inequality therefore gives
\begin{equation}
\begin{aligned}
\mathbb P(M^r\leq0\mid y=+1)
&=\mathbb P(M^r-\Delta_{\mathrm{DE}}^r
\leq-\Delta_{\mathrm{DE}}^r\mid y=+1)\leq\frac{V^r}{V^r+(\Delta_{\mathrm{DE}}^r)^2}.
\end{aligned}
\label{eq:app-clean-margin-tail}
\end{equation}

Similarly, retaining a negative-class example requires its
margin to deviate upward from its mean $-\Delta_{\mathrm{DE}}^r$
by at least $\Delta_{\mathrm{DE}}^r$. Using the common variance
$V^r$ of the two classes, we obtain
\begin{equation}
\operatorname{Recall}^r
\geq\frac{(\Delta_{\mathrm{DE}}^r)^2}{(\Delta_{\mathrm{DE}}^r)^2+V^r},
\qquad
q_{\mathrm{err}}^r
\leq\frac{V^r}{(\Delta_{\mathrm{DE}}^r)^2+V^r}.
\label{eq:app-retention-probability-bounds}
\end{equation}

By Bayes' formula and (\ref{eq:app-retention-probability-bounds}),
\begin{equation}
\operatorname{Precision}^r
=\frac{\operatorname{Recall}^r}
{\operatorname{Recall}^r+\tau q_{\mathrm{err}}^r}
\geq\frac{(\Delta_{\mathrm{DE}}^r)^2}{(\Delta_{\mathrm{DE}}^r)^2+\tau V^r}.
\label{eq:app-precision-bayes}
\end{equation}

Finally, (\ref{eq:app-de-gap}) and the expression for $V^r$ give
\begin{equation}
\frac{(\Delta_{\mathrm{DE}}^r)^2}{V^r}
=\frac{\sin^2\theta^r}{4\sigma^2(1+\sigma^2)}.
\label{eq:app-gap-variance-ratio}
\end{equation}

This yields the two simplified lower bounds in (\ref{eq:app-recall-precision-bounds}).
For fixed $\sigma>0$ and $\tau>0$, both bounds strictly increase
with $\sin^2\theta^r$. Their monotonicity in the alignment-score
difference then follows from Theorem~\ref{thm:de-gap}. When
$\tau=0$, both precision and its lower bound equal one.
\end{proof}

\subsection{Finite-Sample Estimation Error}
\label{app:finite-sample-error}

This section analyzes the effect of finite-sample estimation
on the class matrix and the DE direction of the positive
class.

\begin{theorem}
\label{thm:finite-sample-estimation}
Under the binary model in the main text, fix a space $r\in\mathcal R$.
Suppose that the observed classes $\widetilde y=+1$ and
$\widetilde y=-1$ contain $N_+$ and $N_-$ examples, respectively,
and let $N_{\min}=\min\{N_+,N_-\}$.
There exists a constant $C>0$, depending only on the population
distribution in this space, such that the following holds.
For any $0<\delta<1$, when $N_{\min}$ is sufficiently large and
$\epsilon\geq0$ is sufficiently small, conditional on $N_+,N_-$,
with probability at least $1-\delta$,

\begin{equation}
\lVert\widehat{\bm{A}}_+^r-\bm{A}_+^r\rVert_2
\leq C\left[
\sqrt{\frac{d^r+\log(4/\delta)}{N_+}}
+\frac{d^r+\log(4/\delta)}{N_+}
\right],
\label{eq:app-empirical-moment-error}
\end{equation}

\begin{equation}
\left\|
\widehat{\bm{u}}_{\mathrm{DE}}^r(\widehat{\bm{u}}_{\mathrm{DE}}^r)^\top
-\bm{u}_{\mathrm{DE}}^r(\bm{u}_{\mathrm{DE}}^r)^\top
\right\|_2
 \leq C\left[
\sqrt{\frac{d^r+\log(4/\delta)}{N_{\min}}}
+\frac{d^r+\log(4/\delta)}{N_{\min}}
+\epsilon\right].
\label{eq:app-de-projector-error}
\end{equation}

\end{theorem}

\begin{proof}
Under the Gaussian mixture model in the main text, every unit
projection has sub-Gaussian norm at most $C(1+\sigma)$, and its
centered square is sub-exponential. Applying Bernstein's inequality
on a $1/4$-net of the unit sphere gives
\begin{equation}
\lVert\widehat{\bm{A}}_+^r-\bm{A}_+^r\rVert_2
\leq C(1+\sigma)^2
\left[
\sqrt{\frac{d^r+\log(4/\delta)}{N_+}}
+\frac{d^r+\log(4/\delta)}{N_+}
\right].
\label{eq:app-moment-concentration}
\end{equation}

These bounds hold simultaneously with probability at least
$1-\delta$. Absorbing the fixed $\sigma$ into the constant gives
(\ref{eq:app-empirical-moment-error}). The bound for $\lVert\widehat{\bm{A}}_-^r-\bm{A}_-^r\rVert_2$ follows
analogously.

Since $\bm{B}_+^r=\bm{A}_-^r$ and
$\widehat{\bm{B}}_+^r=\widehat{\bm{A}}_-^r$,
\begin{equation}
\lVert\widehat{\bm{B}}_+^r-\bm{B}_+^r\rVert_2
=\lVert\widehat{\bm{A}}_-^r-\bm{A}_-^r\rVert_2.
\label{eq:app-background-estimation-error}
\end{equation}

The model in the main text and the proof of
Theorem~\ref{thm:de-gap} give $\bm{B}_+^r\succeq\sigma^2 \bm{I}$ and
$\lVert\bm{A}_+^r\rVert_2\leq1+\sigma^2$, with a simple largest generalized
eigenvalue. Construct the population and empirical whitened
matrices as
\begin{equation}
\bm{T}_+^r=(\bm{B}_+^r)^{-1/2}\bm{A}_+^r(\bm{B}_+^r)^{-1/2},\qquad
\widehat{\bm{T}}_+^r
=(\widehat{\bm{B}}_+^r+\epsilon \bm{I})^{-1/2}
\widehat{\bm{A}}_+^r
(\widehat{\bm{B}}_+^r+\epsilon \bm{I})^{-1/2}.
\label{eq:app-whitened-matrices}
\end{equation}

When $\lVert\widehat{\bm{B}}_+^r-\bm{B}_+^r\rVert_2+\epsilon\leq\sigma^2/2$,
the smallest eigenvalue of the empirical denominator matrix is at
least $\sigma^2/2$. The integral representation of the inverse
square root and the resolvent identity give
\begin{equation}
\lVert(\widehat{\bm{B}}_+^r+\epsilon \bm{I})^{-1/2}-(\bm{B}_+^r)^{-1/2}\rVert_2
\leq\frac{\sqrt2}{\sigma^3}
\left(\lVert\widehat{\bm{B}}_+^r-\bm{B}_+^r\rVert_2+\epsilon\right),\qquad
\lVert(\widehat{\bm{B}}_+^r+\epsilon \bm{I})^{-1/2}\rVert_2
\leq\frac{\sqrt2}{\sigma}.
\label{eq:app-inverse-square-root-perturbation}
\end{equation}

In $\widehat{\bm{T}}_+^r-\bm{T}_+^r$, replace the numerator matrix and
then the two inverse square root factors. Taking norms yields
\begin{equation}
\lVert\widehat{\bm{T}}_+^r-\bm{T}_+^r\rVert_2
\leq\frac{2}{\sigma^2}\lVert\widehat{\bm{A}}_+^r-\bm{A}_+^r\rVert_2
 +\frac{4(1+\sigma^2)}{\sigma^4}
\left(\lVert\widehat{\bm{B}}_+^r-\bm{B}_+^r\rVert_2+\epsilon\right).
\label{eq:app-whitened-matrix-error}
\end{equation}

Let $\bm{q}_+^r$ and $\widehat{\bm{q}}_+^r$ be the unit leading
eigenvectors of $\bm{T}_+^r$ and $\widehat{\bm{T}}_+^r$, respectively.
Choose the sign of the empirical eigenvector so that
$(\widehat{\bm{q}}_+^r)^\top \bm{q}_+^r\geq0$. The Davis--Kahan
perturbation bound~\cite{yu2015davis_kahan} gives
\begin{equation}
\lVert\widehat{\bm{q}}_+^r-\bm{q}_+^r\rVert_2
\leq
\frac{2\sqrt2\,\lVert\widehat{\bm{T}}_+^r-\bm{T}_+^r\rVert_2}
{\lambda_1(\bm{T}_+^r)-\lambda_2(\bm{T}_+^r)}.
\label{eq:app-whitened-eigenvector-perturbation}
\end{equation}
Mapping back from the whitened coordinates and normalizing
to unit length gives
$\bm{u}_{\mathrm{DE}}^r=(\bm{B}_+^r)^{-1/2}\bm{q}_+^r/
\lVert(\bm{B}_+^r)^{-1/2}\bm{q}_+^r\rVert_2$, and analogously for the empirical
direction. The inverse square root bound controls the error of
this transformation. For unit vectors, the distance between their
projection matrices is bounded by the distance between the
sign-aligned vectors. Absorbing the fixed background condition
number and eigengap into the constant, we obtain
\begin{equation}
\left\|
\widehat{\bm{u}}_{\mathrm{DE}}^r(\widehat{\bm{u}}_{\mathrm{DE}}^r)^\top
-\bm{u}_{\mathrm{DE}}^r(\bm{u}_{\mathrm{DE}}^r)^\top
\right\|_2
 \leq C\left(
\lVert\widehat{\bm{A}}_+^r-\bm{A}_+^r\rVert_2
+\lVert\widehat{\bm{B}}_+^r-\bm{B}_+^r\rVert_2+\epsilon
\right).
\label{eq:app-de-projector-perturbation}
\end{equation}
Finally, substituting (\ref{eq:app-moment-concentration})--(\ref{eq:app-background-estimation-error}) and using $N_{\min}$ to
control the estimation errors of both the target and background
classes yields (\ref{eq:app-de-projector-error}).
\end{proof}

\section{Implementation Details}
\label{app:validation}

\subsection{Datasets}
\label{app:datasets}

We conduct experiments on seven multimodal datasets. Their details are provided below.

\begin{itemize}
    \item \textbf{UPMC-Food101}~\cite{wang2015recipe} is a multimodal dataset
    for food recognition. It contains approximately 90,000 food images
    and their corresponding textual descriptions collected from websites,
    covering 101 food categories. We adopt the official training/test split~\cite{wang2015recipe}
    and then hold out 10\% of the training set for validation. During
    partitioning, we exclude about 1,500 exact duplicate examples.
    In the main experiments, we use DINOv2-Base~\cite{oquab2023dinov2} and
    XLM-R-base~\cite{conneau2020unsupervised} as the image and text encoders,
    respectively, obtaining a 768-dimensional feature vector for each modality.

    \item \textbf{N24News}~\cite{wang2022n24news} is a multimodal news
    classification dataset constructed from \textit{The New York Times}.
    It contains approximately 60,000 news articles spanning 24 topics,
    each with text and an accompanying image. We use the headline and
    accompanying image as inputs and adopt the standard 8:1:1
    training/validation/test split. We use DINOv2-Base and XLM-R-base as
    the image and text encoders, respectively, obtaining a 768-dimensional
    feature vector for each modality.

    \item \textbf{NWPU-Captions}~\cite{cheng2022nwpu} is a remote sensing
    image captioning dataset containing 31,500 images across 45 scene
    categories, with five textual descriptions per image. We use the scene
    category as the classification target and adopt the official split~\cite{cheng2022nwpu}.
    We use DINOv2-Base and XLM-R-base as the image and text encoders,
    respectively. For each image, we encode its five descriptions
    separately with XLM-R-base. We then take the equally weighted average
    of the five text vectors as the final text feature vector. Each modality
    is represented by a 768-dimensional feature vector.

    \item \textbf{Rakuten France}~\cite{amoualian2021commerce} is a product
    classification dataset from an e-commerce product classification
    challenge. It contains about 80,000 product records across
    27 product categories. We use the product image and text comprising
    the product title and available description as inputs. We adopt
    the standard 8:1:1 training/validation/test split. We use DINOv2-Base
    and XLM-R-base as the image and text encoders, respectively, obtaining
    a 768-dimensional feature vector for each modality.

    \item \textbf{VGGSound}~\cite{chen2020vggsound} is an audio-visual event
    dataset collected from YouTube. We use the 50-class subset released
    by Qin \etal in their IBML work~\cite{qin2026deep}, denoted as
    \textbf{VGGSound50}, which contains approximately 30,000 examples.
    Each example includes audio and video frames from the same video clip.
    We follow the IBML training/test split~\cite{qin2026deep} and hold out 10\% of the
    training set for validation. During partitioning, we exclude
    approximately 2,700 examples from categories outside the target task.
    We use WavLM-Base+~\cite{chen2022wavlm} as the audio encoder.
    We extract representations for all available video frames of each
    clip using DINOv2-Base. We then average these representations with
    equal weights to obtain the image representation. Each modality is
    represented by a 768-dimensional feature vector.

    \item \textbf{MIntRec2.0}~\cite{zhang2024mintrec2} is a multimodal
    intent recognition dataset containing approximately 15,000
    utterances from the television series \textit{Superstore},
    \textit{The Big Bang Theory}, and \textit{Friends}. Each utterance
    includes text, audio, and video. We use the 9,304 in-scope utterances
    belonging to 30 known intent categories and exclude the 5,736
    out-of-scope utterances. We adopt the official in-scope split~\cite{zhang2024mintrec2}.
    We use XLM-R-base, WavLM-Base+, and VideoMAEv2-Base~\cite{wang2023videomae}
    as the text, audio, and video encoders, respectively, obtaining a
    768-dimensional feature vector for each modality.

    \item \textbf{BSD} is constructed from two datasets officially provided
    for the DCASE 2026 Challenge,\footnote{\url{https://dcase.community/challenge2026/task-heterogeneous-audio-classification}}
    BSD35k-CS and BSD10k-v1.2. Both are collected from Freesound and
    contain audio and textual descriptions, comprising about
    35,000 and 10,000 examples, respectively. To align both datasets with
    23 second-level categories of the Broad Sound Taxonomy, we
    remove approximately 2,000 examples belonging to BSD35k-CS's five
    additional \textit{other} categories. The labels in BSD35k-CS are
    assigned by the sound uploaders and contain naturally occurring
    label noise. We use this dataset as the training set. BSD10k-v1.2
    provides expert-annotated and curated labels, and we use it as the
    clean test set. We reuse the precomputed CLAP~\cite{wu2023large} audio and text
    representations provided by the dataset providers, each of which
    is 512-dimensional.
\end{itemize}

\subsection{Baselines}
\label{app:baselines}

This subsection introduces the baselines used in our experiments.

\textbf{(1)~Baselines in the sample-selection experiments.}
We integrate different sample-selection mechanisms into a common
multimodal classifier and optimize the cross-entropy loss with equal
weights on the retained training examples.

\begin{itemize}
    \item \textbf{Standard} is a standard supervised training baseline
    that uses all noisy training examples. It directly minimizes the
    cross-entropy loss between model predictions and observed labels,
    providing a reference for classification performance without sample
    selection.

    \item \textbf{CRUST}~\cite{mirzasoleiman2020coresets} constructs a
    representative training subset based on relationships among the
    gradients of training examples. Its basic idea is to identify
    examples that represent the principal learning directions through
    similarities in gradient space, and to update the model on this
    subset. In our implementation, we use per-example classification
    gradients with respect to the fused representation.

    \item \textbf{AUM}~\cite{pleiss2020identifying}
    uses logit margins over the course of training to identify potentially
    mislabeled examples. For each training example, it computes the
    difference between the logit for the observed label and the largest
    logit among all other classes. This margin is then averaged over
    training. Persistently low values indicate weak model support
    for the observed label.

    \item \textbf{L2D}~\cite{jia2023learning} detects mislabeled examples
    by learning from training dynamics. It forms a time series from the
    predicted probabilities assigned to the observed label across
    training epochs. This sequence is fed into a pretrained LSTM detector to obtain a label-noise score.
    We reuse a fixed detector pretrained on CIFAR-100 to score the
    probability trajectories of the target training examples. We
    determine the number of retained examples using the known noise
    rate.

    \item \textbf{DIST}~\cite{li2023disc} is the dynamic
    instance-specific thresholding mechanism for sample selection in
    DISC. It maintains a separate threshold for each training example,
    updated using an exponential moving average of that example's
    maximum predicted class probability. We integrate this selection
    mechanism into the common classifier and set the decay coefficient
    for threshold updates according to the experimental noise setting.

    \item \textbf{DIST+CT} augments DIST with Confidence Tracking (CT)~\cite{pan2025enhanced}
    to identify correctly labeled yet hard-to-learn examples. CT tracks
    the differences between the predicted probability of the observed
    label and that of each other class. It applies the Mann--Kendall
    test to determine whether these differences exhibit upward trends.
    Following the official setting~\cite{pan2025enhanced}, we take the
    union of the examples supported by CT and those retained by DIST.

    \item \textbf{FINE}~\cite{kim2021fine} selects reliable examples based
    on the alignment between their representations and the principal
    direction within each class. It computes the Gram matrix of the representations within
    each class, and extracts the eigenvector associated with the largest
    eigenvalue. The squared projection of each representation onto this
    direction serves as its alignment score. A two-component GMM is then fitted within each class, and the posterior
    probability of the component with the higher mean is used for
    selection. By default, we use the fused representation immediately
    before the classification head.
\end{itemize}
\textbf{(2)~Co-teaching.}
We use Co-teaching to examine the applicability of REFINE to
cross-selection and updates between networks.

\begin{itemize}
    \item \textbf{Co-teaching}~\cite{han2018co} simultaneously trains
    two separately initialized networks. Each network selects
    small-loss examples from the current minibatch according to its
    own training losses. The selected examples are then passed to the
    other network for parameter updates.
\end{itemize}
\textbf{(3) Semi-supervised learning methods.}
The semi-supervised extension experiments compare the following six
methods.

\begin{itemize}
    \item \textbf{DivideMix}~\cite{li2020dividemix} formulates learning
    with noisy labels as semi-supervised learning. It estimates label
    reliability using a two-component mixture model of the
    loss distribution and partitions the training data into labeled
    and unlabeled subsets. Two networks are trained simultaneously,
    with each network providing the data partition for the other.
    During semi-supervised training, it combines observed labels
    with model predictions to refine the supervision targets for the
    labeled subset and generates pseudo-labels for the unlabeled
    subset. It uses both subsets through a modified MixMatch
    procedure.

    \item \textbf{Mean Teacher}~\cite{tarvainen2017mean} uses unlabeled
    data through prediction consistency between a teacher network and
    a student network. The student is updated by gradient-based
    optimization, whereas the teacher parameters are an exponential
    moving average of the student parameters. Given differently
    perturbed versions of the same input, the student learns to match
    the teacher's prediction targets. This consistency objective is
    optimized jointly with the supervised loss on labeled data.

    \item \textbf{ICT }~\cite{verma2022interpolation} extends consistency
    constraints to interpolations between unlabeled inputs. It linearly
    interpolates two unlabeled inputs and encourages the prediction
    at the interpolated input to match the interpolation of the
    predictions at the original inputs. A teacher model provides the
    target predictions. The interpolation consistency loss is optimized
    jointly with the supervised loss on labeled data, using unlabeled
    data to constrain how the prediction function varies between inputs.

    \item \textbf{CGMatch}~\cite{cheng2025cgmatch} combines current
    prediction confidence with historical prediction stability to
    select unlabeled inputs for semi-supervised training. It uses
    Count-Gap to characterize a persistent prediction preference for
    a particular class. Using a dynamic confidence threshold and a
    Count-Gap threshold, CGMatch divides the unlabeled data into
    easy-to-learn, ambiguous, and hard-to-learn subsets. It applies
    cross-entropy to the easy-to-learn subset and generalized
    cross-entropy to the ambiguous subset.

    \item \textbf{FreeMatch}~\cite{wang2022freematch} adaptively
    adjusts pseudo-label selection thresholds according to the model's
    learning status. It estimates a global threshold using an
    exponential moving average of prediction confidence, then
    constructs class-specific thresholds from class-wise prediction
    statistics. These thresholds dynamically adjust which pseudo-labels
    are used for training as learning progresses. FreeMatch also
    introduces self-adaptive class fairness regularization to encourage
    diverse predictions across classes.

    \item \textbf{DSS+}~\cite{Pan_2026_CVPR} adjusts predicted
    probabilities using the marginal class distribution estimated
    during training. It then selects examples whose adjusted
    predictions agree with their observed labels. It also identifies
    candidate classes from temporal trends in predicted probabilities
    and excludes the competing terms associated with these classes
    from the supervised loss. DSS+ combines cross-selection between
    two networks with mixup-based consistency training using targets
    constructed from model predictions.
\end{itemize}
\textbf{(4) Noise-tolerance loss functions.}
In the loss combination experiments, we use cross-entropy as a reference
and examine the following three losses and their combinations with
REFINE.

\begin{itemize}
    \item \textbf{GCE}~\cite{zhang2018generalized} introduces a tunable parameter
    to construct a family of losses connecting cross-entropy and mean
    absolute error. GCE approaches cross-entropy as this parameter
    approaches zero and is equivalent, up to a constant factor, to
    classification mean absolute error when the parameter equals $1$.
    For intermediate values, GCE adjusts the contribution of examples
    with low predicted probabilities for their observed labels to
    parameter updates. This provides a trade-off between the learning
    ability of cross-entropy and the noise tolerance of mean absolute
    error.

    \item \textbf{SCE}~\cite{wang2019symmetric} is a weighted combination of
    cross-entropy and reverse cross-entropy. The cross-entropy term
    increases the predicted probability of the observed label.
    The reverse cross-entropy term exchanges the roles of the
    prediction and label distributions in cross-entropy, providing
    an additional noise-tolerant constraint during optimization.

    \item \textbf{JAL-CE}~\cite{wang2025joint} is a cross-entropy-based
    instance of Joint Asymmetric Loss. It combines Normalized
    Cross Entropy (NCE) with Asymmetric Mean Square Error (AMSE).
    NCE preserves the learning signal associated with the observed
    label, while AMSE serves as a passive loss that provides an
    additional constraint on the prediction distribution.
\end{itemize}

\subsection{Details of Label Noise Generation}
\label{app:label-noise}

For the six datasets used in the synthetic-noise experiments, we generate
the following four types of label noise separately in the training and
validation sets. Both splits use the same noise type and specified noise
rate, while the test labels remain unchanged.

\begin{itemize}
    \item \textbf{Symmetric noise (Sym)}.
    At the specified noise rate, we select examples uniformly at random
    within each class. We then replace each selected label with a label
    chosen uniformly from the other classes.

    \item \textbf{Asymmetric noise (Asym).}
    To simulate annotation confusion between similar classes, we use the
    same within-class selection rule as Sym. We assign each selected
    example the label of a predefined related class. MIntRec2.0 uses a
    manually specified mapping between intent classes. For the other
    five datasets, we encode class descriptions with XLM-R and retrieve
    the top 10 candidate classes by cosine similarity, excluding the
    source class. We determine the mapping after filtering the candidates
    for semantic compatibility. Among the eligible candidate relationships,
    we prioritize covering more distinct target classes while reducing
    the number of source classes mapped to the same target. Each class
    maps to a class other than itself, and the mapping remains fixed
    across data splits, noise rates, and random seeds.

    \item \textbf{Full-modality instance-dependent noise (F-IDN).}
    Following the noise-generation procedure of TMNR~\cite{zhang2026tmnr},
    we first train a full-modality evidential classifier using the
    original reference labels. The classifier provides class evidence
    and predictive uncertainty for each example. We then select examples
    for label corruption based on their uncertainty, such that examples
    with higher uncertainty are more likely to be selected. For each
    selected example, we assign the label of the class with the highest
    evidence, excluding its original class. Evidence for training examples
    is obtained through five-fold out-of-fold prediction, whereas evidence
    for validation examples is generated by a classifier fitted on the
    entire training set.

    \item \textbf{Partial-modality instance-dependent noise (P-IDN).}
    To further simulate how noisy labels arise in multimodal data, we
    consider annotators who may refer to only some modalities rather
    than all modalities. We first assign each instance a modality
    combination with equal probability. The candidate set includes the
    full modality combination and all subsets obtained by omitting one
    modality. For bimodal data, the candidates comprise one bimodal combination
    and two unimodal combinations. For trimodal data, they comprise the
    full combination and three bimodal combinations. We then obtain the
    corresponding evidence and uncertainty from an evidential classifier
    trained for the assigned modality combination. We select examples
    for corruption and assign incorrect labels using the same rules as
    F-IDN.
\end{itemize}

The main experiments use Sym at a 50\% noise rate, Asym at 40\%, and both
F-IDN and P-IDN at 40\% and 50\%. For each setting, we generate five sets
of noisy labels using different random seeds. All compared methods use
the same noisy labels for each dataset, noise setting, and random seed.

\subsection{Hyperparameter Settings}
\label{app:hyperparameters}

All methods use the same multimodal classifier and shared training settings.
We train the final classifier for 30 epochs using AdamW~\cite{loshchilov2017decoupled},
with an initial learning rate of 10$^{\text{-4}}$, weight decay of 10$^{\text{-3}}$,
and a batch size of 256. The exponential decay rates for the first- and
second-moment estimates in AdamW are set to $0.9$ and $0.999$, respectively.
The learning rate is decayed to zero using a cosine schedule updated at
each optimization step. We clip gradients to a maximum global norm of $1.0$.
The modality-specific projected representations and the fused representation
are all 256-dimensional. The residual fusion module has a hidden dimension
of 512, and the dropout probability is set to $0.3$ in both the projectors
and the fusion module. For our REFINE, we set the parameter $\epsilon$ to
10$^{\text{-4}}$.

In the experiments with unfrozen encoders on UPMC-Food101, the learning rates
for the pretrained encoders and the classification head are set to 10$^{\text{-6}}$
and 10$^{\text{-4}}$, respectively. We jointly train them for 30 epochs using AdamW.
All other settings are identical to those described above.

\begin{algorithm}[!t]
    \caption{Sample Selection with REFINE}
    \label{alg:app-refine-training}
    \small
    \renewcommand{\algorithmicrequire}{\textbf{Input:}}
    \renewcommand{\algorithmicensure}{\textbf{Output:}}
    \begin{algorithmic}[1]
        \raggedright
        \REQUIRE Multimodal model parameters $\bm\theta$, noisy training
            dataset $\mathcal D$, number of classes $K$, regularization
            parameter $\epsilon>0$, total number of training epochs
            $\mathrm{MaxEpoch}$, number of warm-up epochs
            $\mathrm{WarmUpEpoch}$, selection interval
            $\mathrm{SelectInterval}$.
        \ENSURE Model parameters $\bm\theta$.
        \STATE $\bm\theta\leftarrow
            \mathrm{WarmUp}(\mathcal D,\bm\theta,\mathrm{WarmUpEpoch})$.
        \STATE $t\leftarrow\mathrm{WarmUpEpoch}$.
        \WHILE{$t<\mathrm{MaxEpoch}$}
            \STATE Initialize the selected subset
                $\mathcal C\leftarrow\varnothing$.
            \STATE Extract representations
                $\{\bm z_i^r\}_{i=1,\ldots,N,\,r\in\mathcal R}$
                for all training inputs using the current model.
            \FOR{$k=1,\ldots,K$}
                \FOR{$r\in\mathcal R$}
                    \STATE Compute the background weights and background
                        matrix $\bm B_k^r$.
                    \STATE Solve for the discriminative eigenvector
                        $\bm u_{k,\mathrm{DE}}^r$.
                \ENDFOR
                \FORALL{unordered pairs $\{r,s\}\subseteq\mathcal R$,
                    $r\neq s$}
                    \STATE Construct two-dimensional projection
                        representations $\{\bm p_{i,k}^{r,s}\}_{i=1}^{N}$
                        for all training inputs.
                    \STATE Solve for the two-dimensional discriminative
                        eigenvector $\bm a_k^{r,s}=(a_k^r,a_k^s)^\top$.
                    \STATE Record a win for the space with the larger
                        coefficient magnitude and a loss for the other,
                        ties count as neither wins nor losses.
                \ENDFOR
                \STATE Retain spaces with at least as many wins as losses
                    to form $\mathcal R_k$.
            \ENDFOR
            \FOR{$k=1,\ldots,K$}
                \FOR{$r\in\mathcal R_k$}
                    \STATE For each example with observed label $k$,
                        compute its alignment scores for all classes,
                        $s_{i,j}^r=[(\bm u_{j,\mathrm{DE}}^r)^\top
                        \bm z_i^r]^2$.
                    \STATE Retain examples satisfying
                        $s_{i,k}^r>\max_{j\neq k}s_{i,j}^r$
                        to form $\mathcal C_k^r$.
                \ENDFOR
            \ENDFOR
            \STATE $\mathcal C\leftarrow
                \bigcup_{k=1}^{K}\bigcup_{r\in\mathcal R_k}\mathcal C_k^r$.
            \STATE Continue training $\bm\theta$ with cross-entropy loss
                on the examples in $\mathcal C$ and their observed labels
                through epoch
                $\min(t+\mathrm{SelectInterval},\mathrm{MaxEpoch})$.
            \STATE $t\leftarrow
                \min(t+\mathrm{SelectInterval},\mathrm{MaxEpoch})$.
        \ENDWHILE
        \RETURN $\bm\theta$.
    \end{algorithmic}
\end{algorithm}

\subsection{Details of Sample-Selection Experiments}
\label{app:sample-selection}

We compare different sample-selection methods using the common multimodal
classifier and shared training settings described above. For REFINE, we
first use all noisy training examples for 10 warm-up epochs and then
perform sample selection every five epochs, for a total of 30 training
epochs. Algorithm~\ref{alg:app-refine-training} presents the complete
procedure.

\begin{algorithm}[!t]
    \caption{FreeMatch with REFINE}
    \label{alg:app-refine-freematch}
    \small
    \renewcommand{\algorithmicrequire}{\textbf{Input:}}
    \renewcommand{\algorithmicensure}{\textbf{Output:}}
    \begin{algorithmic}[1]
        \raggedright
        \REQUIRE Multimodal model parameters $\bm\theta$, noisy training
            dataset $\mathcal D$, REFINE regularization parameter
            $\epsilon>0$, augmentation operators
            $\alpha_{\mathrm w},\alpha_{\mathrm s}$, loss weights
            $\lambda_{\mathrm u},\lambda_{\mathrm f}$, total number of
            training epochs $\mathrm{MaxEpoch}$, number of warm-up epochs
            $\mathrm{WarmUpEpoch}$, selection interval
            $\mathrm{SelectInterval}$.
        \ENSURE Model parameters $\bm\theta$.
        \STATE $\bm\theta\leftarrow
            \mathrm{WarmUp}(\mathcal D,\bm\theta,\mathrm{WarmUpEpoch})$.
        \STATE Initialize FreeMatch's thresholds and prediction statistics,
            and set $t\leftarrow\mathrm{WarmUpEpoch}$.
        \WHILE{$t<\mathrm{MaxEpoch}$}
            \STATE $\mathcal C\leftarrow
                \mathrm{REFINE}(\mathcal D,\bm\theta,\epsilon)$
                (Algorithm~\ref{alg:refine-selection}).
            \STATE Construct the labeled subset
                $\mathcal X\leftarrow\mathcal C$ and the unlabeled subset
                $\mathcal U\leftarrow
                \{\bm x_i:(\bm x_i,\widetilde y_i)
                \in\mathcal D\setminus\mathcal C\}$.
            \FOR{$e=t+1,\ldots,
                \min(t+\mathrm{SelectInterval},\mathrm{MaxEpoch})$}
                \FORALL{minibatches $\mathcal B_{\ell},\mathcal B_{\mathrm u}$
                    from $\mathcal X$ and $\mathcal U$, respectively}
                    \STATE Compute the supervised cross-entropy loss
                        $\mathcal L_{\mathrm s}$ using the observed labels
                        and weak-view predictions of $\mathcal B_{\ell}$.
                    \STATE For inputs in $\mathcal B_{\mathrm u}$, compute
                        weak-view predictions $\bm q_i\leftarrow
                        p_{\bm\theta}(\cdot\mid
                        \alpha_{\mathrm w}(\bm x_i))$ and generate
                        pseudo-labels $\widehat y_i\leftarrow
                        \operatorname*{arg\,max}_k q_i(k)$.
                    \STATE Update FreeMatch's prediction statistics and
                        class-specific self-adaptive thresholds $\tau(k)$
                        using $\{\bm q_i\}$.
                    \STATE Compute strong-view predictions
                        $\bm Q_i\leftarrow p_{\bm\theta}(\cdot\mid
                        \alpha_{\mathrm s}(\bm x_i))$.
                    \STATE Use pseudo-labels satisfying
                        $q_i(\widehat y_i)\geq\tau(\widehat y_i)$ to
                        supervise the corresponding strong-view predictions,
                        and compute $\mathcal L_{\mathrm u}$ as in FreeMatch.
                    \STATE Compute the self-adaptive class fairness
                        regularization term $\mathcal L_{\mathrm f}$
                        as in FreeMatch.
                    \STATE $\mathcal L\leftarrow
                        \mathcal L_{\mathrm s}
                        +\lambda_{\mathrm u}\mathcal L_{\mathrm u}
                        +\lambda_{\mathrm f}\mathcal L_{\mathrm f}$.
                    \STATE Update model parameters $\bm\theta$
                        using $\mathcal L$.
                \ENDFOR
            \ENDFOR
            \STATE $t\leftarrow
                \min(t+\mathrm{SelectInterval},\mathrm{MaxEpoch})$.
        \ENDWHILE
        \RETURN $\bm\theta$.
    \end{algorithmic}
\end{algorithm}

\begin{algorithm}[!t]
    \caption{GCE with REFINE}
    \label{alg:app-refine-gce}
    \small
    \renewcommand{\algorithmicrequire}{\textbf{Input:}}
    \renewcommand{\algorithmicensure}{\textbf{Output:}}
    \begin{algorithmic}[1]
        \raggedright
        \REQUIRE Multimodal model parameters $\bm\theta$, noisy training
            dataset $\mathcal D$, GCE parameter $q$, REFINE regularization
            parameter $\epsilon>0$, total number of training epochs
            $\mathrm{MaxEpoch}$, number of warm-up epochs
            $\mathrm{WarmUpEpoch}$, selection interval
            $\mathrm{SelectInterval}$.
        \ENSURE Model parameters $\bm\theta$.
        \STATE Warm up on all training examples in $\mathcal D$ with
            GCE loss for $\mathrm{WarmUpEpoch}$ epochs, updating
            $\bm\theta$.
        \STATE $t\leftarrow\mathrm{WarmUpEpoch}$.
        \WHILE{$t<\mathrm{MaxEpoch}$}
            \STATE $\mathcal C\leftarrow
                \mathrm{REFINE}(\mathcal D,\bm\theta,\epsilon)$
                (Algorithm~\ref{alg:refine-selection}).
            \STATE Continue training $\bm\theta$ with GCE loss on the
                examples in $\mathcal C$ and their observed labels
                through epoch
                $\min(t+\mathrm{SelectInterval},\mathrm{MaxEpoch})$.
            \STATE $t\leftarrow
                \min(t+\mathrm{SelectInterval},\mathrm{MaxEpoch})$.
        \ENDWHILE
        \RETURN $\bm\theta$.
    \end{algorithmic}
\end{algorithm}

\subsection{Details of Semi-Supervised Experiments}
\label{app:semi-supervised}

Using the same multimodal classifier and training settings, we integrate
our algorithm into FreeMatch~\cite{wang2022freematch} and
DSS+~\cite{Pan_2026_CVPR}, respectively, and compare with existing
semi-supervised methods. We use FreeMatch as an example to describe how
REFINE is combined with semi-supervised learning. After warm-up, REFINE
periodically partitions the noisy training dataset: the selected examples,
together with their observed labels, form the labeled subset, and the
remaining inputs form the unlabeled subset. FreeMatch then uses the
labeled subset for supervised training and the unlabeled subset for
pseudo-label learning, yielding the combined method R-FreeMatch.

\begin{table}[!t]
  \caption{Macro-F1 (\%, mean$\pm$std.) of the proposed method and baseline
  methods on different datasets under different noise types and ratios. The
  ``Noise'' columns show the noise type and ratio, and the ``Average'' column
  reports the mean across the six noise settings. The best and second-best
  results are highlighted in \textbf{boldface} and \underline{underlined}, respectively.}
  \label{tab:app-macro-f1}
  \centering
  \small
  \renewcommand{\arraystretch}{1.12}
  \newcommand{\mfscore}[2]{#1\,\textpm\,#2}
  \newcommand{\bestmfscore}[2]{\textbf{#1\,\textpm\,#2}}
  \newcommand{\secondmfscore}[2]{\underline{#1\,\textpm\,#2}}
  \settowidth{\dimen0}{\textbf{NWPU-Captions}}
  \settowidth{\dimen2}{DIST+CT}
  \settowidth{\dimen4}{\bestmfscore{00.00}{0.00}}
  \settowidth{\dimen6}{\textbf{F-IDN 50\%}}
  \setlength{\tabcolsep}{\dimexpr
    (\textwidth-\dimen0-\dimen2-3\dimen4-4\dimen6-3\arrayrulewidth)/18\relax}
  \setlength{\aboverulesep}{0pt}
  \setlength{\belowrulesep}{0pt}
  \newlength{\appMFinset}
  \setlength{\appMFinset}{1.8pt}
  \newcommand{\appMFtopgap}{0pt}
  \newcommand{\appMFbottomgap}{\appMFinset}
  \newcommand{\appMFgaps}[2]{\noalign{%
    \gdef\appMFtopgap{#1}\gdef\appMFbottomgap{#2}}}
  \newcommand{\appMFseparator}{\smash{\rule[
    \dimexpr-\dp\csname @arstrutbox\endcsname+\appMFbottomgap\relax
  ]{\arrayrulewidth}{\dimexpr
    \ht\csname @arstrutbox\endcsname+\dp\csname @arstrutbox\endcsname-
    \appMFtopgap-\appMFbottomgap\relax}}}
  \newcommand{\appMFheaderseparator}{\smash{\rule[
    -\dimexpr\dp\csname @arstrutbox\endcsname+\cmidrulewidth\relax
  ]{\arrayrulewidth}{\dimexpr
    \ht\csname @arstrutbox\endcsname+\dp\csname @arstrutbox\endcsname+
    \cmidrulewidth-\appMFinset\relax}}}
  \newlength{\appMFrefineoffset}
  \setlength{\appMFrefineoffset}{\dimexpr\dimen0+2\tabcolsep+\arrayrulewidth\relax}
  \newcommand{\appMFrefinerowshade}{\noalign{\begingroup
    \color{gray!15}%
    \vskip\appMFinset
    \hbox{\hskip\appMFrefineoffset\vrule width\dimexpr\textwidth-\appMFrefineoffset\relax
      height0pt depth\dimexpr
        \ht\csname @arstrutbox\endcsname+\dp\csname @arstrutbox\endcsname-2\appMFinset\relax}%
    \vskip-\dimexpr
      \ht\csname @arstrutbox\endcsname+\dp\csname @arstrutbox\endcsname-\appMFinset\relax
    \endgroup}}
  \begin{tabular}{c!{\appMFseparator}c!{\appMFseparator}cccccc!{\appMFseparator}c}
    \toprule
    \multicolumn{1}{c!{\appMFheaderseparator}}{\tableheaderspan{\textbf{Dataset}}} &
    \multicolumn{1}{c!{\appMFheaderseparator}}{\tableheaderspan{\textbf{Method}}} &
    \multicolumn{6}{c!{\appMFheaderseparator}}{\tableheadercell{\textbf{Noise}}} &
    \tableheaderspan{\textbf{Average}} \\
    \cmidrule(lr){3-8}
    & & \tableheadercell{\textbf{Sym 50\%}} & \tableheadercell{\textbf{Asym 40\%}} &
    \tableheadercell{\textbf{F-IDN 40\%}} & \tableheadercell{\textbf{F-IDN 50\%}} &
    \tableheadercell{\textbf{P-IDN 40\%}} & \tableheadercell{\textbf{P-IDN 50\%}} & \\
    \midrule

    \appMFgaps{\appMFinset}{0pt}
    \multirow{8}{*}{\textbf{UPMC-Food101}}
    & Standard
    & \mfscore{82.63}{0.34}
    & \mfscore{67.76}{0.24}
    & \mfscore{84.30}{0.10}
    & \mfscore{82.62}{0.10}
    & \mfscore{84.39}{0.15}
    & \mfscore{82.75}{0.06}
    & \mfscore{80.74}{0.05} \\
    \appMFgaps{0pt}{0pt}
    & CRUST
    & \mfscore{84.71}{0.17}
    & \secondmfscore{81.93}{0.23}
    & \mfscore{83.11}{0.23}
    & \mfscore{82.08}{0.05}
    & \mfscore{82.60}{0.19}
    & \mfscore{81.74}{0.18}
    & \mfscore{82.69}{0.05} \\
    & AUM
    & \mfscore{83.04}{0.12}
    & \mfscore{65.89}{0.59}
    & \mfscore{83.73}{0.14}
    & \mfscore{82.06}{0.29}
    & \mfscore{82.08}{0.27}
    & \mfscore{80.14}{0.16}
    & \mfscore{79.49}{0.14} \\
    & L2D
    & \mfscore{84.39}{0.22}
    & \mfscore{72.80}{0.37}
    & \secondmfscore{85.70}{0.20}
    & \mfscore{84.04}{0.12}
    & \secondmfscore{85.47}{0.20}
    & \secondmfscore{83.90}{0.32}
    & \mfscore{82.72}{0.06} \\
    & DIST
    & \secondmfscore{85.45}{0.22}
    & \mfscore{80.78}{0.38}
    & \mfscore{85.54}{0.22}
    & \mfscore{84.23}{0.16}
    & \mfscore{84.37}{0.18}
    & \mfscore{82.40}{0.13}
    & \secondmfscore{83.80}{0.11} \\
    & DIST+CT
    & \secondmfscore{85.45}{0.21}
    & \mfscore{80.78}{0.38}
    & \mfscore{85.54}{0.22}
    & \secondmfscore{84.24}{0.11}
    & \mfscore{84.37}{0.18}
    & \mfscore{82.42}{0.12}
    & \secondmfscore{83.80}{0.09} \\
    \appMFgaps{0pt}{\appMFinset}
    & FINE
    & \mfscore{84.69}{0.27}
    & \mfscore{68.87}{0.60}
    & \mfscore{84.68}{0.09}
    & \mfscore{83.58}{0.12}
    & \mfscore{83.15}{0.16}
    & \mfscore{81.94}{0.13}
    & \mfscore{81.15}{0.13} \\
    \cmidrule{2-9}
    \appMFgaps{\appMFinset}{\appMFinset}
    \appMFrefinerowshade
    & \textbf{REFINE}
    & \bestmfscore{86.48}{0.21}
    & \bestmfscore{85.11}{0.34}
    & \bestmfscore{86.70}{0.07}
    & \bestmfscore{85.71}{0.13}
    & \bestmfscore{86.28}{0.14}
    & \bestmfscore{85.00}{0.25}
    & \bestmfscore{85.88}{0.05} \\
    \midrule

    \appMFgaps{\appMFinset}{0pt}
    \multirow{8}{*}{\textbf{MIntRec2.0}}
    & Standard
    & \mfscore{19.75}{1.02}
    & \bestmfscore{21.20}{0.57}
    & \secondmfscore{25.49}{0.29}
    & \secondmfscore{24.53}{0.86}
    & \secondmfscore{24.60}{0.87}
    & \secondmfscore{22.39}{1.62}
    & \secondmfscore{22.99}{0.34} \\
    \appMFgaps{0pt}{0pt}
    & CRUST
    & \mfscore{18.91}{0.92}
    & \mfscore{18.62}{2.10}
    & \mfscore{24.37}{0.70}
    & \mfscore{23.74}{1.27}
    & \mfscore{23.16}{1.29}
    & \mfscore{20.83}{2.08}
    & \mfscore{21.60}{0.81} \\
    & AUM
    & \mfscore{13.98}{0.97}
    & \mfscore{13.98}{0.36}
    & \mfscore{17.59}{0.57}
    & \mfscore{17.00}{0.99}
    & \mfscore{17.21}{1.02}
    & \mfscore{15.50}{0.98}
    & \mfscore{15.88}{0.56} \\
    & L2D
    & \mfscore{18.18}{1.01}
    & \mfscore{16.68}{0.68}
    & \mfscore{20.94}{1.04}
    & \mfscore{18.88}{0.92}
    & \mfscore{20.38}{0.67}
    & \mfscore{17.83}{0.93}
    & \mfscore{18.81}{0.74} \\
    & DIST
    & \secondmfscore{20.80}{0.87}
    & \mfscore{13.09}{0.76}
    & \mfscore{24.66}{0.38}
    & \mfscore{24.29}{0.79}
    & \mfscore{23.79}{0.95}
    & \mfscore{22.21}{1.35}
    & \mfscore{21.48}{0.40} \\
    & DIST+CT
    & \secondmfscore{20.80}{0.87}
    & \mfscore{13.04}{0.70}
    & \mfscore{24.66}{0.38}
    & \mfscore{24.29}{0.79}
    & \mfscore{23.79}{0.95}
    & \mfscore{22.21}{1.35}
    & \mfscore{21.47}{0.41} \\
    \appMFgaps{0pt}{\appMFinset}
    & FINE
    & \mfscore{20.31}{0.97}
    & \mfscore{18.50}{1.20}
    & \mfscore{23.81}{2.27}
    & \mfscore{24.04}{0.78}
    & \mfscore{24.08}{0.31}
    & \mfscore{19.86}{0.74}
    & \mfscore{21.77}{0.47} \\
    \cmidrule{2-9}
    \appMFgaps{\appMFinset}{\appMFinset}
    \appMFrefinerowshade
    & \textbf{REFINE}
    & \bestmfscore{21.41}{0.74}
    & \secondmfscore{20.61}{0.78}
    & \bestmfscore{25.62}{1.22}
    & \bestmfscore{25.38}{1.10}
    & \bestmfscore{24.68}{0.49}
    & \bestmfscore{23.10}{1.83}
    & \bestmfscore{23.47}{0.37} \\
    \midrule

    \appMFgaps{\appMFinset}{0pt}
    \multirow{8}{*}{\textbf{NWPU-Captions}}
    & Standard
    & \mfscore{98.30}{0.16}
    & \mfscore{77.75}{0.85}
    & \mfscore{97.23}{0.17}
    & \mfscore{95.12}{0.47}
    & \mfscore{98.12}{0.27}
    & \mfscore{96.71}{0.49}
    & \mfscore{93.87}{0.25} \\
    \appMFgaps{0pt}{0pt}
    & CRUST
    & \secondmfscore{99.19}{0.03}
    & \mfscore{94.65}{0.23}
    & \mfscore{97.14}{0.39}
    & \mfscore{95.82}{0.56}
    & \mfscore{97.45}{0.32}
    & \mfscore{96.48}{0.35}
    & \mfscore{96.79}{0.21} \\
    & AUM
    & \mfscore{98.89}{0.22}
    & \mfscore{77.77}{0.84}
    & \mfscore{98.28}{0.17}
    & \mfscore{96.63}{0.47}
    & \mfscore{97.68}{0.20}
    & \mfscore{96.14}{0.35}
    & \mfscore{94.23}{0.18} \\
    & L2D
    & \mfscore{98.63}{0.14}
    & \mfscore{83.56}{2.30}
    & \mfscore{98.39}{0.17}
    & \mfscore{97.15}{0.50}
    & \secondmfscore{98.48}{0.35}
    & \secondmfscore{97.71}{0.24}
    & \mfscore{95.65}{0.35} \\
    & DIST
    & \mfscore{99.14}{0.06}
    & \secondmfscore{97.18}{0.15}
    & \mfscore{98.58}{0.19}
    & \mfscore{97.11}{0.44}
    & \mfscore{97.89}{0.27}
    & \mfscore{96.65}{0.29}
    & \secondmfscore{97.76}{0.14} \\
    & DIST+CT
    & \mfscore{99.14}{0.06}
    & \mfscore{97.16}{0.17}
    & \mfscore{98.58}{0.19}
    & \mfscore{97.11}{0.44}
    & \mfscore{97.89}{0.27}
    & \mfscore{96.65}{0.29}
    & \mfscore{97.75}{0.14} \\
    \appMFgaps{0pt}{\appMFinset}
    & FINE
    & \mfscore{98.84}{0.17}
    & \mfscore{87.82}{1.16}
    & \secondmfscore{98.70}{0.16}
    & \secondmfscore{98.10}{0.32}
    & \mfscore{98.12}{0.33}
    & \mfscore{96.94}{0.52}
    & \mfscore{96.42}{0.25} \\
    \cmidrule{2-9}
    \appMFgaps{\appMFinset}{\appMFinset}
    \appMFrefinerowshade
    & \textbf{REFINE}
    & \bestmfscore{99.55}{0.10}
    & \bestmfscore{98.39}{0.41}
    & \bestmfscore{99.55}{0.11}
    & \bestmfscore{99.37}{0.09}
    & \bestmfscore{99.34}{0.15}
    & \bestmfscore{99.12}{0.08}
    & \bestmfscore{99.22}{0.08} \\
    \midrule

    \appMFgaps{\appMFinset}{0pt}
    \multirow{8}{*}{\textbf{VGGSound50}}
    & Standard
    & \mfscore{68.26}{1.04}
    & \mfscore{57.31}{1.15}
    & \mfscore{69.82}{0.37}
    & \mfscore{68.54}{0.44}
    & \mfscore{69.70}{0.55}
    & \mfscore{68.36}{0.93}
    & \mfscore{67.00}{0.49} \\
    \appMFgaps{0pt}{0pt}
    & CRUST
    & \mfscore{69.91}{0.49}
    & \mfscore{59.60}{2.73}
    & \bestmfscore{70.70}{0.68}
    & \bestmfscore{69.67}{0.98}
    & \secondmfscore{70.44}{0.67}
    & \secondmfscore{69.16}{1.58}
    & \mfscore{68.25}{0.58} \\
    & AUM
    & \secondmfscore{70.86}{0.16}
    & \mfscore{60.45}{2.46}
    & \mfscore{69.66}{0.54}
    & \mfscore{67.98}{0.60}
    & \mfscore{69.91}{0.70}
    & \mfscore{68.22}{0.88}
    & \mfscore{67.85}{0.56} \\
    & L2D
    & \mfscore{68.25}{0.75}
    & \mfscore{60.34}{1.58}
    & \mfscore{69.28}{0.62}
    & \mfscore{67.26}{0.59}
    & \mfscore{69.02}{0.97}
    & \mfscore{67.08}{1.15}
    & \mfscore{66.87}{0.19} \\
    & DIST
    & \mfscore{69.94}{0.44}
    & \secondmfscore{63.78}{0.95}
    & \secondmfscore{70.07}{0.28}
    & \mfscore{67.74}{0.80}
    & \mfscore{69.75}{0.56}
    & \mfscore{68.45}{0.90}
    & \secondmfscore{68.29}{0.32} \\
    & DIST+CT
    & \mfscore{69.94}{0.44}
    & \mfscore{63.69}{0.76}
    & \secondmfscore{70.07}{0.28}
    & \mfscore{67.74}{0.80}
    & \mfscore{69.75}{0.56}
    & \mfscore{68.45}{0.90}
    & \mfscore{68.27}{0.29} \\
    \appMFgaps{0pt}{\appMFinset}
    & FINE
    & \mfscore{70.56}{0.84}
    & \mfscore{56.91}{0.68}
    & \mfscore{69.93}{0.70}
    & \mfscore{68.67}{0.62}
    & \mfscore{70.13}{0.57}
    & \mfscore{68.50}{1.06}
    & \mfscore{67.45}{0.43} \\
    \cmidrule{2-9}
    \appMFgaps{\appMFinset}{\appMFinset}
    \appMFrefinerowshade
    & \textbf{REFINE}
    & \bestmfscore{71.06}{0.50}
    & \bestmfscore{64.67}{1.03}
    & \mfscore{69.97}{0.68}
    & \secondmfscore{69.01}{0.48}
    & \bestmfscore{70.57}{0.80}
    & \bestmfscore{69.40}{0.58}
    & \bestmfscore{69.11}{0.38} \\
    \midrule

    \appMFgaps{\appMFinset}{0pt}
    \multirow{8}{*}{\textbf{Rakuten France}}
    & Standard
    & \mfscore{68.91}{0.29}
    & \mfscore{56.91}{1.08}
    & \mfscore{69.18}{0.70}
    & \mfscore{66.15}{0.22}
    & \mfscore{70.64}{0.21}
    & \mfscore{68.24}{0.28}
    & \mfscore{66.67}{0.30} \\
    \appMFgaps{0pt}{0pt}
    & CRUST
    & \mfscore{69.11}{0.35}
    & \mfscore{58.24}{1.11}
    & \mfscore{69.00}{0.58}
    & \mfscore{66.19}{0.23}
    & \mfscore{70.58}{0.27}
    & \mfscore{68.40}{0.42}
    & \mfscore{66.92}{0.17} \\
    & AUM
    & \mfscore{71.36}{0.59}
    & \secondmfscore{60.77}{0.65}
    & \mfscore{70.69}{0.32}
    & \mfscore{67.80}{0.42}
    & \mfscore{70.93}{0.53}
    & \mfscore{68.26}{0.31}
    & \secondmfscore{68.30}{0.12} \\
    & L2D
    & \mfscore{69.11}{0.53}
    & \mfscore{59.18}{0.41}
    & \mfscore{70.07}{0.46}
    & \secondmfscore{67.92}{0.63}
    & \mfscore{70.55}{0.75}
    & \mfscore{68.80}{0.57}
    & \mfscore{67.61}{0.30} \\
    & DIST
    & \bestmfscore{71.66}{0.27}
    & \mfscore{56.50}{0.66}
    & \secondmfscore{70.73}{0.47}
    & \mfscore{67.58}{0.47}
    & \secondmfscore{71.07}{0.23}
    & \bestmfscore{69.08}{0.49}
    & \mfscore{67.77}{0.21} \\
    & DIST+CT
    & \secondmfscore{71.61}{0.25}
    & \mfscore{56.49}{0.64}
    & \bestmfscore{70.79}{0.50}
    & \mfscore{67.59}{0.48}
    & \bestmfscore{71.15}{0.32}
    & \secondmfscore{68.95}{0.46}
    & \mfscore{67.76}{0.21} \\
    \appMFgaps{0pt}{\appMFinset}
    & FINE
    & \mfscore{69.29}{0.63}
    & \mfscore{57.85}{0.33}
    & \mfscore{69.00}{0.58}
    & \mfscore{66.19}{0.23}
    & \mfscore{70.58}{0.27}
    & \mfscore{68.37}{0.43}
    & \mfscore{66.88}{0.18} \\
    \cmidrule{2-9}
    \appMFgaps{\appMFinset}{\appMFinset}
    \appMFrefinerowshade
    & \textbf{REFINE}
    & \mfscore{70.72}{0.50}
    & \bestmfscore{63.62}{0.63}
    & \mfscore{70.32}{0.20}
    & \bestmfscore{68.53}{0.23}
    & \mfscore{70.60}{0.18}
    & \mfscore{68.82}{0.80}
    & \bestmfscore{68.77}{0.23} \\
    \midrule

    \appMFgaps{\appMFinset}{0pt}
    \multirow{8}{*}{\textbf{N24News}}
    & Standard
    & \mfscore{57.14}{0.48}
    & \mfscore{47.44}{0.80}
    & \mfscore{60.73}{0.34}
    & \mfscore{59.58}{0.39}
    & \mfscore{61.21}{0.30}
    & \mfscore{59.21}{0.35}
    & \mfscore{57.55}{0.15} \\
    \appMFgaps{0pt}{0pt}
    & CRUST
    & \secondmfscore{59.68}{0.29}
    & \secondmfscore{53.16}{0.44}
    & \bestmfscore{61.34}{0.79}
    & \secondmfscore{60.12}{0.28}
    & \secondmfscore{61.40}{0.34}
    & \secondmfscore{60.30}{0.86}
    & \secondmfscore{59.33}{0.41} \\
    & AUM
    & \mfscore{58.75}{0.60}
    & \mfscore{49.60}{0.21}
    & \mfscore{60.57}{0.33}
    & \mfscore{59.33}{0.50}
    & \mfscore{60.40}{0.65}
    & \mfscore{58.55}{0.50}
    & \mfscore{57.87}{0.16} \\
    & L2D
    & \mfscore{56.89}{0.15}
    & \mfscore{49.71}{0.83}
    & \mfscore{60.14}{0.32}
    & \mfscore{58.68}{0.35}
    & \mfscore{60.76}{0.21}
    & \mfscore{58.72}{0.47}
    & \mfscore{57.48}{0.20} \\
    & DIST
    & \mfscore{59.50}{0.51}
    & \secondmfscore{53.16}{0.33}
    & \bestmfscore{61.34}{0.52}
    & \mfscore{59.82}{0.45}
    & \mfscore{60.92}{0.53}
    & \mfscore{59.83}{0.30}
    & \mfscore{59.10}{0.21} \\
    & DIST+CT
    & \mfscore{59.50}{0.51}
    & \mfscore{53.11}{0.48}
    & \secondmfscore{61.28}{0.49}
    & \mfscore{59.82}{0.45}
    & \mfscore{60.92}{0.54}
    & \mfscore{59.85}{0.27}
    & \mfscore{59.08}{0.23} \\
    \appMFgaps{0pt}{\appMFinset}
    & FINE
    & \mfscore{57.61}{0.64}
    & \mfscore{50.18}{0.43}
    & \mfscore{60.91}{0.59}
    & \mfscore{59.80}{0.25}
    & \mfscore{61.23}{0.21}
    & \mfscore{59.18}{0.26}
    & \mfscore{58.15}{0.16} \\
    \cmidrule{2-9}
    \appMFgaps{\appMFinset}{\appMFinset}
    \appMFrefinerowshade
    & \textbf{REFINE}
    & \bestmfscore{61.52}{0.62}
    & \bestmfscore{58.39}{0.24}
    & \mfscore{61.20}{0.43}
    & \bestmfscore{60.21}{0.54}
    & \bestmfscore{61.42}{0.57}
    & \bestmfscore{60.35}{0.53}
    & \bestmfscore{60.52}{0.26} \\
    \bottomrule
  \end{tabular}
\end{table}

For each unlabeled input, we generate a pseudo-label from the prediction
on its weakly augmented view. Pseudo-labels that pass FreeMatch's
self-adaptive thresholds are used to supervise predictions on the
strongly augmented views. The training objective comprises the supervised
loss $\mathcal L_{\mathrm s}$ on the labeled subset, the pseudo-label
consistency loss $\mathcal L_{\mathrm u}$ on the unlabeled subset, and
FreeMatch's self-adaptive class fairness regularization term
$\mathcal L_{\mathrm f}$~\cite{wang2022freematch}.
Algorithm~\ref{alg:app-refine-freematch} presents the complete procedure,
where $\alpha_{\mathrm w}$ and $\alpha_{\mathrm s}$ denote weak and strong
augmentation, respectively, and $p_{\bm\theta}(\cdot\mid\bm x)$ denotes
the vector of class probabilities predicted by the model for input
$\bm x$.

\subsection{Combination with Noise-Robust Loss Functions}
\label{app:robust-losses}

Using the same multimodal classifier and training settings, we compare
the classification performance of GCE~\cite{zhang2018generalized},
SCE~\cite{wang2019symmetric}, and JAL-CE~\cite{wang2025joint}, each with
and without REFINE. Taking GCE as an example, we first warm up the model
on all noisy training examples using GCE. We then periodically apply
REFINE and continue training with GCE on the selected examples and their
observed labels. Algorithm~\ref{alg:app-refine-gce} presents this
procedure. SCE and JAL-CE follow the same combination procedure, with
the corresponding loss used during both warm-up and subsequent training.

\section{Additional Results}
\label{app:additional-results}

We report additional results for the experiments on synthetic noise in terms of
Macro-F1. Table~\ref{tab:app-macro-f1} presents the results
on six datasets under six synthetic label-noise settings, using the same
experimental setup as in the main text. Consistent with the accuracy
results in the main text, REFINE achieves the highest average Macro-F1 on all
six datasets.

\section{Reproducibility Statement}
\label{app:reproducibility}
We provide implementation details, involving illustrative algorithm in Appendix~\ref{app:sample-selection}.
The source code will be publicly released for reproducibility.

\clearpage
\bibliography{main}
\bibliographystyle{unsrt}

\end{document}